\documentclass{article}

\usepackage[main, final]{neurips_2026}

\usepackage[utf8]{inputenc}
\usepackage[T1]{fontenc}
\usepackage{hyperref}
\usepackage{url}
\usepackage{booktabs}
\usepackage{amsfonts}
\usepackage{amssymb}
\usepackage{amsmath}
\usepackage{amsthm}
\usepackage{mathtools}
\usepackage{nicefrac}
\usepackage{microtype}
\usepackage{xcolor}
\usepackage{algorithm}
\usepackage{algpseudocode}
\usepackage{bm}
\usepackage{multirow}
\usepackage{subcaption}
\usepackage{enumitem}
\usepackage{graphicx}
\usepackage{wrapfig}

\newtheorem{theorem}{Theorem}
\newtheorem{proposition}{Proposition}

\newcommand{\EE}{\mathbb{E}}

\newcommand{\cN}{\mathcal{N}}

\newcommand{\cT}{\mathcal{T}}

\newcommand{\htilde}{\tilde{h}}
\newcommand{\dbar}{\bar{d}}

\title{HoTS: Homophily-Aware Temperature Scaling for Graph Neural Network Calibration}

\author{%
  Inwoo Tae \\
  UNIST \\
  {\small\texttt{vitainu0104@unist.ac.kr}}
  \And
  Yoontae Hwang\thanks{Co-corresponding authors.} \\
  Pusan National University \\
  {\small\texttt{yoontae.hwang@pusan.ac.kr}}
  \And
  Yongjae Lee\footnotemark[1] \\
  UNIST \\
  LinqAlpha \\
  {\small\texttt{yongjaelee@unist.ac.kr}}
}

\begin{document}
\maketitle

\begin{abstract}
For graph node classification, calibrated class probabilities are needed when confidence scores, usually the maximum predicted class probability, are used to rank predictions, defer uncertain nodes to human review, or control risk. Existing post-hoc calibrators either apply one global temperature or use graph-aware modules without a principled structural form. We study how local graph structure should enter node-level calibration. Our first results show that a logit-only temperature rule is insufficient when nodes with identical logits but different local homophily require different optimal temperatures. We then analyze a population-concentration contextual stochastic block model with Gaussian features and a one-layer linear GCN. Under equidistant class means, the Bayes posterior over class-template scores is a temperature-scaled softmax whose inverse-temperature is governed by a homophily-dependent signal strength. In the positive-signal homophilic regime, the resulting temperature decreases approximately inversely with normalized local homophily. This law motivates Homophily-aware Temperature Scaling (HoTS), a simple post-hoc calibrator that assigns each node a positive scalar temperature from entropy-based logit concentration and estimated local homophily. HoTS has three temperature parameters, preserves the predicted class, and learns the strength of the structural correction from calibration data. Across 18 node-classification benchmarks, two GNN backbones, and eight calibration baselines, HoTS achieves the best mean Expected Calibration Error (ECE) of $4.79\%$, the best average rank, and the most reliable confidence ranking in selective classification. Code is available at \url{https://github.com/inu0104/HoTS}.
\end{abstract}

\section{Introduction}
Graph neural networks (GNNs) are widely used for node classification on relational data~\citep{kipf2017semi, hamilton2017inductive}. In graph learning systems, probabilities often guide ranking, triage, human review, and risk scoring, not only class labels. Calibration, the agreement between predictive confidence and empirical correctness, is therefore central to reliability~\citep{naeini2015obtaining}. Post-hoc calibration improves these probabilities without retraining the base model~\citep{niculescu2005predicting}. Temperature scaling is the standard approach, using a single positive scalar to rescale logits while preserving accuracy~\citep{guo2017calibration}, but it assumes that miscalibration is sufficiently uniform across examples. This assumption is fragile on graphs. Through message passing, a node prediction depends on both its own features and the realized composition of its neighborhood. Thus, two nodes with the same logits or confidence can have different probabilities of being correct. A homophilic neighborhood may support a confident prediction, whereas a mixed or heterophilic neighborhood may make the same confidence unreliable~\citep{zhu2020beyond}. Since these regimes can coexist in one graph, a global temperature must compromise across structurally different nodes. Recent graph-aware calibration methods address this issue with node-wise temperatures or graph-dependent logit transformations using auxiliary GNNs, attention, spectral features, or ensemble-style architectures~\citep{wang2021confident, hsu2022makes, zhuang2025gets, xie2024exploring}. They show that graph structure matters, but do not specify the form of the structural correction.

This work asks \emph{how the optimal node-wise temperature should depend on local graph structure, given that neighborhood composition affects the reliability of GNN confidence}. 
We show that logits alone are insufficient. Two nodes with the same logits can require different optimal temperatures when their local homophily provides different class-consistent support. Under this condition, homophily must enter the calibrator to avoid a residual calibration gap. We analyze a population-concentration Contextual Stochastic Block Model (CSBM) with Gaussian node features and a one-layer linear GCN with self-loops. CSBM permits an explicit Bayes posterior calculation while preserving the effect of neighborhood composition on the aggregated signal. Under equidistant class means, the Bayes posterior for class-template scores depends on local homophily through a homophily-dependent signal strength. In the positive-signal regime, stronger homophilic support sharpens the posterior, while weaker support softens it. In the strongly homophilic regime, the Bayes temperature decreases approximately in inverse proportion to normalized local homophily, so well-supported nodes need less correction than structurally ambiguous ones. Because the Bayes result uses oracle homophily, we convert it into a practical rule using estimated local homophily and entropy-based logit concentration. An errors-in-variables argument shows that noisy homophily estimates attenuate the inverse-homophily relationship, motivating us to learn the homophily sensitivity.

These results motivate Homophily-aware Temperature Scaling (HoTS), a node-wise calibrator. HoTS uses a global offset, a correction scale, and a homophily-sensitivity exponent to assign each node a positive scalar temperature from entropy-based logit concentration and estimated local homophily, preserving the predicted class whenever the top logit is unique. Across 18 datasets and eight baselines, HoTS achieves the lowest mean ECE of $4.79\%$ and the most reliable confidence ranking in selective classification. Our contributions are to establish the necessity of homophily for optimal graph calibration, derive a CSBM Bayes calibration law linking temperature to local homophily, translate this law into a simple post-hoc method, and validate it across diverse benchmarks.

\section{Related Work} \label{sec:related}

\paragraph{Post-hoc calibration}
Post-hoc calibration maps model scores to reliable probabilities using a held-out calibration set while keeping the trained predictor fixed~\citep{niculescu2005predicting, naeini2015obtaining}. Classical methods include Platt scaling for binary classification~\citep{platt1999probabilistic} and multiclass extensions via score decomposition or per-class probability estimation~\citep{zadrozny2002transforming}. For deep neural networks, temperature scaling has become a standard baseline because it learns only one scalar temperature and preserves the predicted label whenever the top logit is unique~\citep{guo2017calibration}. Its simplicity makes it robust and easy to deploy, but limits its ability to correct instance-dependent calibration errors. More expressive calibrators address this limitation by learning richer transformations of logits or probabilities. Vector and matrix scaling introduce class-dependent or affine logit transformations~\citep{guo2017calibration}, while beta calibration and Dirichlet calibration capture asymmetric and class-specific distortions~\citep{kull2017beta, kull2019beyond}. Other approaches use ensembles, confidence-dependent temperatures, or entropy-based temperature functions~\citep{zhang2020mix, balanya2024adaptive}. These methods are effective in standard classification settings, but they usually treat examples as independent calibration instances and do not model how relational structure affects confidence reliability. A related issue is whether calibration should preserve the original prediction. General logit transformations can change the predicted class and mix calibration with refinement, whereas order-preserving calibrators avoid this~\citep{rahimi2020intra}. Label-wise reliability reranking also addresses residual correctness differences after calibration while preserving predicted labels and class probabilities~\citep{tae2026post}. HoTS follows this principle by assigning each node a positive scalar temperature, preserving the predicted class while allowing the correction strength to vary with local graph structure.

\paragraph{Calibration for graph neural networks} Calibration for graph neural networks is more challenging because node predictions are coupled through message passing. Empirical studies show that GNN calibration varies substantially across datasets and nodes, and that standard post-hoc methods often fail to account for graph-specific reliability patterns~\citep{teixeira2019graph}. Several works therefore incorporate graph structure into calibration. CaGCN~\citep{wang2021confident} learns a topology-aware calibration function with an auxiliary GCN and exploits the observation that confidence patterns can be homophilic. RBS~\citep{liu2022calibration} groups nodes using a same-class-neighbor ratio and learns bin-wise temperatures. GATS~\citep{hsu2022makes} analyzes sources of GNN miscalibration, including nodewise predictive diversity, distance to training nodes, relative confidence, and neighborhood similarity, and uses an attention-based model for node-wise temperature scaling. Recent work further expands the structural signals used for graph calibration. SimCalib~\citep{tang2024simcalib} uses global and local nodewise similarity signals, ASTS~\citep{xie2024exploring} studies heterophily-aware calibration through spectral information, WATS~\citep{li2025calibrating} uses graph wavelet features for node-specific temperature scaling, and GETS~\citep{zhuang2025gets} combines input and model ensembles through a graph mixture-of-experts framework. Other methods improve calibration at training time or by modifying the graph itself. GCL~\citep{wang2022gcl} introduces a graph calibration loss, Moderate Message Passing~\citep{wang2024moderate} mitigates confidence bias induced by message passing, DCGC~\citep{yang2024calibrating} improves calibration through data-centric edge weighting, and \citet{fang2024improving} emphasize that calibration should preserve discriminative ability between correct and incorrect predictions. These studies establish that graph structure is crucial for calibration. However, most existing methods introduce structure through learned architectures, handcrafted features, spectral filters, wavelets, ensembles, or training objectives. They do not derive what functional relationship should hold between local graph structure and the optimal node-wise temperature. Our work fills this gap by deriving a homophily-dependent temperature law and translating it into a simple calibrator.

\paragraph{Graph structure theory and CSBM analysis} The role of homophily and heterophily in GNNs has been widely studied for node classification. Prior work shows that standard message passing can struggle under heterophily and motivates architectures that better separate ego and neighbor information or exploit higher-order neighborhoods~\citep{zhu2020beyond}. Other studies show that the relationship between homophily and GNN performance is more nuanced than a global homophily measure suggests~\citep{ma2021homophily}. Recent work further emphasizes local homophily and shows that mismatch between local and global structure can create performance differences across node groups~\citep{loveland2024performance}. These findings motivate our focus on node-level structural reliability. The contextual stochastic block model provides a principled framework for analyzing how graph structure and node features jointly affect GNN behavior. Foundational CSBM work characterizes thresholds for recovery in attributed graphs~\citep{lu2023contextual}. Recent analyses study Bayes-optimal inference and the performance of graph convolutional architectures under CSBM assumptions~\citep{duranthon2023optimal, dalle2024optimal}. These works clarify how structure, features, and message passing affect classification accuracy and representation quality. Our work is complementary to this theoretical literature. Existing CSBM analyses mainly study detection, classification, separability, or representation behavior, whereas calibration asks whether predicted probabilities match true correctness probabilities. To our knowledge, no prior work derives the Bayes-optimal temperature as an explicit function of local homophily under a CSBM. HoTS connects this CSBM Bayes posterior calculation to a practical node-wise temperature scaling rule.
\section{Derivation of Node-wise Temperature}
\paragraph{Assumptions.}
The theory is developed under a uniform-prior CSBM with Gaussian node features and a one-layer linear GCN with self-loops. For the exact posterior calculation, we use a population-concentration version of the model: conditional on the realized oracle local homophily $h_i$, the degree and neighbor class composition are fixed at their population values, so finite degree and composition fluctuations are not included in the posterior calculation. The class means are assumed to be equidistant, so that after centering they form a regular simplex on their affine span. Throughout the theoretical sections, $h_i$ denotes the oracle local homophily defined using true labels. It is used only to characterize the Bayes signal and is not assumed to be available at test time. The implemented method uses an estimated homophily $\hat h_i$ from an auxiliary GCN supervised using observed labels. The formal model and posterior calculation are in Appendix~\ref{app:formal_csbm} and Appendix~\ref{app:bayes_proof}.


\paragraph{Theoretical Motivation} \label{sec:theory} The theory has three roles. \textbf{First}, it shows that homophily is not merely a useful heuristic: if the Bayes-optimal temperature varies with $h$ after conditioning on logits, then excluding $h$ causes a positive calibration gap. \textbf{Second}, under the population-concentration CSBM, it computes the Bayes posterior whose inverse-temperature is controlled by a homophily-dependent signal strength. \textbf{Third}, it turns this group-level Bayes signal into the HoTS form by adding a node-wise entropy proxy and a learnable exponent that accounts for noisy homophily estimation.

\paragraph{Why Homophily Is Necessary for GNN Calibration} Let $\tau=1/T$ denote inverse-temperature. For fixed logits $z$ and label $y$, define the temperature-scaled cross-entropy loss as $\ell(\tau;z,y)=-\log\mathrm{softmax}(\tau z)_y$. Let $\cT_z$ be the class of positive inverse-temperature calibrators depending only on $z$, and let $\cT_{z,h}$ be the corresponding class of calibrators that may depend on both $z$ and $h$.

\begin{proposition}[Necessity of homophily]\label{prop:necessity}
Let $\tau^*(z,h)$ be the CE-optimal positive inverse-temperature conditional on $(Z=z,H=h)$, and let $\tau^\star(z)$ be the CE-optimal positive inverse-temperature conditional only on $Z=z$. Under the local regularity conditions stated in Appendix~\ref{app:necessity_proof}, if
\begin{equation}
    \operatorname{Var}_H\!\left(\tau^*(z,H)\mid Z=z\right)>0    
\end{equation}

on a set of positive measure, then
\begin{equation}
    \inf_{\tau\in\cT_{z,h}} \EE[\ell(\tau(Z,H);Z,Y)] <  \inf_{\tau\in\cT_z} \EE[\ell(\tau(Z);Z,Y)] .    
\end{equation}
Moreover, in the local quadratic regime, this improvement is proportional to the curvature-weighted conditional variance of $\tau^*(Z,H)$ given $Z$.
\end{proposition}

\begin{proof}
The proof is given in Appendix~\ref{app:necessity_proof}.
\end{proof}

Intuitively, if two nodes have the same logits but different homophily values, and the Bayes-optimal temperature differs between them, then no logit-only calibrator can be optimal for both. Homophily is therefore necessary whenever it explains residual variation in the optimal positive temperature after conditioning on logits. Empirically, calibrators incorporating estimated local homophily (such as HoTS) outperform logit-only calibrators (TS, ETS) in aggregate across the 18 benchmarks (Section~\ref{exp:main}), which provides direct evidence for this gap. Proposition~\ref{prop:necessity} reduces the question to whether the Bayes-optimal positive temperature genuinely varies with $h$ within fixed-logit groups. The next result verifies the homophily dependence of the Bayes inverse-temperature for $Z=S$, where $S$ denotes the CSBM class-template score vector, on the positive-signal region; it does not claim the same exact posterior for arbitrary learned logits. It computes the corresponding signed Bayes inverse-temperature, and ordinary positive-temperature statements are restricted to the positive-signal region $\mathcal H_+$. Thus Proposition~\ref{prop:necessity} applies here with the logit variable specialized to the class-template score vector $Z=S$, provided that $H$ remains non-degenerate within fixed-$S$ groups, as formalized in the strictness argument in Appendix~\ref{app:bayes_proof}.

\paragraph{CSBM Bayes Posterior and Homophily-Dependent Temperature} Let $m_c=\mu_c-\bar\mu$ and $\rho_0=\|m_c\|$. Under the equidistant-means assumption stated in Appendix~\ref{app:formal_csbm}, the centered class means form a regular simplex on their affine span. To avoid confusion with the learned GCN logits, we denote the theoretical centered class-template scores by $s_c=\rho_0^{-1}m_c^\top(\bar x_i-\bar\mu)$. The exact posterior result below is a statement about these class-template scores, not about an arbitrary learned linear head $W$; learned logits are connected to the theory only through the later surrogate HoTS form. We write the normalized homophily as $\tilde h=\frac{Kh-1}{K-1},$ so that $\tilde h=0$ corresponds to a class-random neighborhood and $\tilde h>0$ corresponds to homophily above the class-random baseline. Finally, let $v=\sigma^2/(\dbar+1)$ and $\Gamma(h)=\rho_0(1+\dbar\tilde h)/(\dbar+1)$.

\begin{theorem}[CSBM Bayes posterior for class-template scores]\label{thm:signal}
Under the population-concentration CSBM described above and formalized in Appendix~\ref{app:formal_csbm}, the Bayes posterior conditional on the realized homophily $H_i=h$ and the class-template score vector $s$ is 
\begin{equation}\label{eq:bayes_posterior}
    P(Y_i=c\mid s,H_i=h)
    =
    \mathrm{softmax}\!\left(
        \frac{\Gamma(h)}{v}s
    \right)_c .
\end{equation}
Hence, for these class-template scores, the signed Bayes inverse-temperature is $\tau_{\mathrm{Bayes}}(h)=\Gamma(h)/v$.
\end{theorem}
\begin{proof}
The proof is given in Appendix~\ref{app:bayes_proof}.
\end{proof}
Theorem~\ref{thm:signal} says that, under the population-concentration CSBM and for the class-template logits, local homophily changes the Bayes posterior only through the scalar signal strength $\Gamma(h)$. Consequently, for non-constant class-template scores $s$ and $h\in\mathcal H_+$, the CE-optimal positive inverse-temperature conditional on $(s,h)$ is $\Gamma(h)/v$. When $\Gamma(h)>0$, this signed Bayes inverse-temperature corresponds to an ordinary positive temperature. When $\Gamma(h)<0$, the Bayes rule reverses the template-logit ordering and cannot be represented by positive temperature scaling. Since $\Gamma(h)$ is increasing in $h$, more homophilic neighborhoods induce a larger Bayes inverse-temperature, and therefore a sharper posterior, whenever the signal is positive. On the positive-signal region $\mathcal H_+=\{h:\Gamma(h)>0\}$, the corresponding positive Bayes temperature is
\begin{equation}\label{eq:group_temperature}
    T_{\mathrm{Bayes}}(h)
    =
    \frac{v}{\Gamma(h)}
    =
    \frac{\sigma^2}{(1+\dbar\htilde)\rho_0}.
\end{equation}
In the large-$\dbar\htilde$ homophilic regime, $\Gamma(h)\approx\Gamma_0\htilde$, where $\Gamma_0=\rho_0\dbar/(\dbar+1)$. Thus $T_{\mathrm{Bayes}}(h)\propto\htilde^{-1}$, which motivates the homophily-dependent denominator of the HoTS temperature defined in Section~\ref{sec:method}.

\paragraph{From Bayes Temperature to a Node-wise Surrogate} \label{sec:method_form} The Bayes temperature in \eqref{eq:group_temperature} is group-level: all nodes with the same homophily receive the same temperature. To translate this group-level law into a node-wise rule, we use a reliability-matching surrogate. Consider binary classification, where node $i$ has logits $z_i=(z_{i1},z_{i2})$. When the predicted class is $1$, $z_{i1}>z_{i2}$, so the positive logit gap $\delta_i=z_{i1}-z_{i2}$ is positive and $\sigma(\delta_i)\in(1/2,1)$ is the model's confidence in the predicted class. Temperature scaling preserves this range: for any $T_i>0$, $\sigma(\delta_i/T_i)\in(1/2,1)$ as well. Calibration asks that this confidence equal the true correctness probability for node $i$. Since the per-node correctness probability is unobservable, we replace it with a group-level rate. Theorem~\ref{thm:signal} shows that the Bayes temperature depends on local homophily, so we group nodes by $h$ and use the homophily-conditional correctness target $q(h_i)=\Pr(\hat y_i=y_i\mid H_i=h_i)$. Matching the calibrated confidence to this target, $\sigma(\delta_i/T_i)=q(h_i)$, and applying $\operatorname{logit}=\sigma^{-1}$ to both sides gives
\begin{equation}\label{eq:match_T}
    T_i^{\mathrm{match}} = \frac{\delta_i}{\operatorname{logit}(q(h_i))}.
\end{equation}
This is not an exact pointwise CE optimum under the full CSBM posterior; it is a practical rule that separates node confidence (through $\delta_i$) from structural reliability (through $q(h_i)$). Equation~\eqref{eq:match_T} is well-defined and yields a positive temperature only when $q(h_i)>1/2$; this corresponds to the positive-signal regime $h\in\mathcal H_+$ of Theorem~\ref{thm:signal}, in which the Bayes posterior is sharper than chance. The implemented HoTS rule extends the form of this surrogate to all nodes via the absolute-value denominator $(|\hat{\tilde{h}}_i| + \epsilon)^\alpha$ (with $\hat{\tilde{h}}_i$ the normalized estimated homophily and $\epsilon > 0$ a small stabilizer) for stability when the estimated structural signal may be weak or negative, rather than as an exact pointwise approximation outside $H_+$.

For multi-class logits, we replace the binary gap by the entropy-based concentration proxy $\sqrt{2K \log K (1-e_i)}$, where $e_i$ is the normalized predictive entropy, obtained from the local softmax-entropy expansion around the uniform distribution. Under the binary centered class-template CSBM, the group correctness probability satisfies $q(h)=\Phi(\Gamma(h)/\sqrt v)$. In the moderate-margin and large-$\dbar\htilde$ homophilic regime, this gives $\operatorname{logit}(q(h))\approx\lambda\htilde$ for a positive constant $\lambda$. The entropy expansion and the binary reliability calculation are given in Appendix~\ref{app:entropy_expansions} and Appendix~\ref{app:binary_reliability}.

Combining these approximations yields, for $\htilde_i>0$ with $\dbar\htilde_i\gg 1$, the positive-temperature oracle ratio form
\begin{equation}\label{eq:oracle_ratio}
    T_i^{\mathrm{match}} \approx c\, \frac{\sqrt{2K\log K(1-e_i)}}{\htilde_i}, \qquad c>0,
\end{equation}
valid in the moderate-margin, large-$\dbar\htilde$ part of the positive-signal homophilic regime. This oracle ratio is not intended for $\htilde_i\le 0$; the absolute-value denominator in the implemented HoTS rule is a stable positive-temperature parametrization rather than an exact signed Bayes posterior approximation. Thus, temperature increases with raw logit concentration and decreases with structural reliability in this approximation. The practical form replaces the oracle $\htilde_i$ by the estimated $\hat{\htilde}_i$ and introduces a learnable exponent $\alpha$ on the homophily denominator $(|\hat{\htilde}_i|+\varepsilon)^\alpha$ to account for estimation noise and deviations from the ideal CSBM regime.

\paragraph{Why the Exponent Is Learnable} \label{sec:learnable} The large-$\dbar\htilde$ approximation to the CSBM Bayes temperature gives the inverse-homophily law $T\propto\htilde^{-1}$, corresponding to the oracle exponent $\alpha=1$ in that regime. In practice, however, the oracle $\htilde$ is replaced by a noisy estimate $\hat{\htilde}$. This weakens the observed log-log relationship between temperature and homophily, so the exponent should be learned rather than fixed.

\begin{proposition}[Exponent attenuation]\label{prop:alpha}
Consider the regime where the positive-temperature oracle is approximated by an inverse-magnitude law. Let $u=\log(|\htilde|+\varepsilon)$ and let the observed log-homophily be $\hat u=u+\eta$, where $\eta$ is zero-mean noise independent of $u$ with variance $\tau^2$. If the oracle relation is $\log T^*=a-u$, but we fit the surrogate $\log T_\alpha=a_\alpha-\alpha\hat u$ by population least squares, then

\begin{equation}\label{eq:alpha_attenuation}
    \alpha^*
    =
    \frac{\operatorname{Var}(u)}
    {\operatorname{Var}(u)+\tau^2}
    \leq 1 .
\end{equation}

The inequality is strict whenever $\tau^2>0$ and $\operatorname{Var}(u)>0$.
\end{proposition}

\begin{proof}
The proof is given in Appendix~\ref{app:alpha_proof}.
\end{proof}

Proposition~\ref{prop:alpha} is exact for the log-space least-squares surrogate, not for the full CE objective used to train HoTS.
On CSBM, the fitted $\alpha$ decreases from $1.009$ ($\tau=0$, oracle recovery) to $0.182$ ($\tau=0.5$), confirming the errors-in-variables attenuation (Figure~\ref{fig:theory_validation}b). Fixing $\alpha=1$ on real data degrades mean ECE by $+1.10$ percentage points (Appendix~\ref{app:component_ablation}).

\section{Homophily-Aware Temperature Scaling} \label{sec:method}
\paragraph{Method} We propose Homophily-aware Temperature Scaling (HoTS), a node-wise calibration method that combines predictive uncertainty with local graph structure. The theory suggests two design principles. First, in the large-$\dbar\htilde$ part of the positive-signal homophilic CSBM regime, the Bayes temperature decreases approximately inversely with normalized homophily. Second, node-level confidence variation can be captured by an entropy-based logit-concentration proxy. For each node $i$, HoTS defines the temperature
\begin{equation}\label{eq:gts}
    {T_i = T_{\mathrm{base}} + \frac{\beta\sqrt{2K\log K(1-e_i)}}{(|\hat{\htilde}_i|+\varepsilon)^\alpha}}
\end{equation}
where $e_i=H(\mathrm{softmax}(z_i))/\log K$ is the normalized predictive entropy, and
$\hat{\htilde}_i=(K\hat h_i-1)/(K-1)$ is the normalized estimated homophily. The calibrated prediction is
$\mathrm{softmax}(z_i/T_i)$. Equation~\eqref{eq:gts} acts on the learned GCN logits $z_i$, whereas Theorem~\ref{thm:signal} characterizes the Bayes posterior with respect to the CSBM class-template scores $s$; HoTS thus uses the learned logits as a practical surrogate for the unobserved class-template scores, and the form of the temperature is motivated by, rather than derived from, the exact Bayes rule. Section~\ref{exp:main} reports the main calibration results across 18 benchmarks.

\paragraph{Homophily estimation} The method uses the estimated homophily $\hat h_i$ rather than the oracle homophily $h_i$. In our implementation, $\hat h_i$ is produced by a small auxiliary GCN trained on targets computed using only training and validation labels, excluding self-loops and unobserved-label neighbors (architecture and training details in Appendix~\ref{app:hots_predictor}). The predictor is held fixed during temperature fitting. This construction avoids test-label leakage while retaining local graph information.

\paragraph{Learnable parameters} The learnable temperature parameters are $T_{\mathrm{base}}>0$, $\beta>0$, and $\alpha>0$.
The base temperature $T_{\mathrm{base}}$ provides a global positive offset, $\beta$ controls the
magnitude of the homophily-aware correction, and $\alpha$ controls the sensitivity to the
estimated homophily magnitude. The stabilizer $\varepsilon>0$ is fixed. The CSBM analysis
(Proposition~\ref{prop:alpha}) suggests an asymptotic exponent in $[0,1]$ under the log-space
least-squares surrogate, but we do not enforce this upper bound in implementation, since real
graphs may deviate from the idealized regime. In implementation, the positivity constraints on
$T_{\mathrm{base}}$, $\beta$, and $\alpha$ are enforced through softplus parameterizations.

\paragraph{Interpretation and prediction preservation} The numerator $\sqrt{2K\log K(1-e_i)}$ increases with logit concentration and is small near the uniform predictive distribution. In the oracle positive-signal homophilic regime, a denominator $(\htilde_i+\varepsilon)^\alpha$ is motivated by the large-$\dbar\htilde$ inverse-homophily Bayes temperature law, not by the exact finite-$\dbar\htilde$ Bayes expression near $\htilde_i=0$. In the implemented HoTS rule, we replace the oracle $\htilde_i$ by the estimated signal and use $(|\hat{\htilde}_i|+\varepsilon)^\alpha$ as a stable positive-temperature parametrization when the estimated structural signal may be negative, weak, or noisy. The offset $T_{\mathrm{base}}$ and stabilizer $\varepsilon$ absorb the finite baseline behavior that is not captured by the large-signal inverse law. This absolute-value extension should not be interpreted as an exact signed Bayes posterior approximation in the negative-signal regime. In the idealized CSBM regime with oracle homophily, the asymptotic exponent is $\alpha=1$. With estimated homophily and model mismatch, the exponent is learned from calibration data, allowing the strength of the homophily correction to adapt to each dataset. Since HoTS uses positive temperatures, it preserves the original predicted class and adjusts only the confidence assigned to the prediction.

\section{Empirical Validation} \label{sec:experiments}
The empirical study addresses three questions. We ask whether a homophily-aware temperature rule improves calibration across graph regimes without a high-capacity calibration network, whether controlled CSBM experiments recover the predicted inverse dependence of temperature on homophily and its attenuation under noisy homophily estimates, and whether HoTS yields confidence rankings that are more useful for downstream decisions.

\paragraph{Experimental Setup} \label{exp:setup} The evaluation covers $18$ public node-classification benchmarks spanning citation graphs (Cora, CiteSeer, PubMed)~\citep{yang2016revisiting}, Amazon co-purchase and Coauthor graphs (Computers, Photo, CS, Physics, CoraFull)~\citep{shchur2018pitfalls}, WebKB and Wikipedia-derived heterophilic graphs (Texas, Cornell, Wisconsin, Chameleon, Squirrel, Actor)~\citep{pei2020geom}, the heterophilous benchmark of~\citet{platonov2023critical} (Roman-Empire, tolokers), and large-scale graphs ogbn-arxiv~\citep{hu2020open} and Reddit~\citep{hamilton2017inductive}. Following convention in prior GNN and calibration benchmarks, Table~\ref{tab:main_ece} groups these datasets into homophilic, heterophilic, and large-scale regimes; the heterophilic group follows the benchmark designations of \citet{pei2020geom} and ~\citet{platonov2023critical} rather than a strict homophily threshold, and per-dataset homophily statistics are reported in Appendix~\ref{app:datasets}. Unless otherwise stated, each dataset and backbone pair is run with $10$ random seeds under a $20/10/70$ train, validation, and test split. The backbones are GCN~\citep{kipf2017semi} and GAT~\citep{velivckovic2017graph}. Table~\ref{tab:main_ece} pools both backbones, while Appendix~\ref{app:per_backbone} gives the per-backbone breakdown. The primary metric is Expected Calibration Error (ECE) with $15$ equally spaced confidence bins. Because ECE is a binned summary, Appendix~\ref{app:additional_metrics} also reports negative log-likelihood (NLL) and degree-stratified ECE. Baselines include standard post-hoc calibrators, Temperature Scaling (TS), Vector Scaling (VS)~\citep{guo2017calibration}, Ensemble Temperature Scaling (ETS)~\citep{zhang2020mix}, entropy-aware HTS~\citep{balanya2024adaptive}, and graph-aware calibrators CaGCN, GATS, GETS, and WATS. RBS and SimCalib are omitted only because their public reference could not be located, while Section~\ref{sec:related} discusses them as the closest prior works to HoTS. HoTS fits its three temperature parameters on the calibration split, uses training-split cross-entropy for early stopping, and uses the homophily estimate described in Section~\ref{sec:method}. 

\begin{table}[t]
\centering
\setlength{\tabcolsep}{2pt}
\caption{Calibration ECE (\%, mean$\pm$std over 10 seeds per backbone, pooled across GCN and GAT). Best per row in \textcolor{red}{red}, second in \textbf{bold}, third \underline{underlined} (excluding Uncal).}

\label{tab:main_ece}

\resizebox{\textwidth}{!}{%
\begin{tabular}{lcccccccccc}
\toprule
Dataset & Uncal & TS & VS & ETS & HTS & CaGCN & GATS & GETS & WATS & HoTS (ours) \\
\midrule
\multicolumn{11}{l}{\textit{Homophilic}} \\
Cora & $19.64 \pm 2.62$ & $3.14 \pm 0.82$ & $\boldsymbol{2.90 \pm 0.78}$ & $\underline{2.94 \pm 0.82}$ & $3.58 \pm 0.56$ & $4.12 \pm 0.78$ & $2.99 \pm 0.96$ & $3.86 \pm 0.82$ & $3.08 \pm 0.73$ & $\textcolor{red}{2.50 \pm 0.72}$ \\
CiteSeer & $18.57 \pm 4.19$ & $\boldsymbol{3.04 \pm 0.62}$ & $3.20 \pm 0.56$ & $\textcolor{red}{2.93 \pm 0.63}$ & $3.78 \pm 0.86$ & $4.81 \pm 0.79$ & $3.19 \pm 0.86$ & $5.08 \pm 1.07$ & $\underline{3.10 \pm 0.57}$ & $3.50 \pm 0.80$ \\
PubMed & $12.15 \pm 2.15$ & $1.04 \pm 0.28$ & $1.16 \pm 0.39$ & $\underline{1.00 \pm 0.31}$ & $1.22 \pm 0.65$ & $1.11 \pm 0.99$ & $\boldsymbol{0.88 \pm 0.28}$ & $1.14 \pm 0.38$ & $\textcolor{red}{0.83 \pm 0.30}$ & $1.26 \pm 0.30$ \\
CoraFull & $32.93 \pm 4.51$ & $2.76 \pm 0.81$ & $2.84 \pm 0.63$ & $\textcolor{red}{2.20 \pm 0.63}$ & $3.94 \pm 0.63$ & $\boldsymbol{2.56 \pm 0.61}$ & $2.77 \pm 0.80$ & $3.72 \pm 0.98$ & $3.85 \pm 1.65$ & $\underline{2.73 \pm 0.81}$ \\
Computers & $5.86 \pm 1.42$ & $2.19 \pm 0.45$ & $2.07 \pm 0.45$ & $2.12 \pm 0.50$ & $1.78 \pm 0.55$ & $\textcolor{red}{1.39 \pm 0.34}$ & $1.91 \pm 0.49$ & $2.66 \pm 0.73$ & $\boldsymbol{1.62 \pm 0.60}$ & $\underline{1.76 \pm 0.49}$ \\
Photo & $4.31 \pm 0.90$ & $\underline{1.23 \pm 0.31}$ & $1.27 \pm 0.26$ & $1.30 \pm 0.49$ & $\textcolor{red}{1.14 \pm 0.48}$ & $1.35 \pm 0.42$ & $1.24 \pm 0.32$ & $1.42 \pm 0.42$ & $\underline{1.23 \pm 0.39}$ & $\boldsymbol{1.21 \pm 0.44}$ \\
CS & $2.72 \pm 1.20$ & $2.26 \pm 0.83$ & $\underline{2.03 \pm 0.68}$ & $2.22 \pm 0.87$ & $\textcolor{red}{1.34 \pm 0.31}$ & $2.39 \pm 0.83$ & $2.11 \pm 0.67$ & $2.64 \pm 2.82$ & $3.25 \pm 0.59$ & $\boldsymbol{1.51 \pm 0.60}$ \\
Physics & $1.20 \pm 0.63$ & $0.87 \pm 0.29$ & $0.85 \pm 0.26$ & $0.82 \pm 0.31$ & $\textcolor{red}{0.52 \pm 0.20}$ & $\underline{0.72 \pm 0.26}$ & $\underline{0.72 \pm 0.29}$ & $1.00 \pm 0.65$ & $1.41 \pm 0.47$ & $\boldsymbol{0.54 \pm 0.23}$ \\
\midrule
\multicolumn{11}{l}{\textit{Heterophilic}} \\
Texas & $21.86 \pm 8.25$ & $21.64 \pm 8.49$ & $18.23 \pm 6.93$ & $19.27 \pm 8.41$ & $\boldsymbol{17.69 \pm 6.71}$ & $20.33 \pm 7.82$ & $\underline{17.85 \pm 7.36}$ & $18.28 \pm 4.98$ & $24.78 \pm 8.18$ & $\textcolor{red}{15.37 \pm 5.90}$ \\
Cornell & $20.13 \pm 8.13$ & $20.05 \pm 7.39$ & $19.65 \pm 6.58$ & $18.18 \pm 7.60$ & $\underline{18.06 \pm 7.39}$ & $\boldsymbol{15.93 \pm 7.01}$ & $18.58 \pm 7.91$ & $19.58 \pm 7.90$ & $25.17 \pm 6.87$ & $\textcolor{red}{14.28 \pm 5.82}$ \\
Wisconsin & $19.31 \pm 7.31$ & $19.06 \pm 7.19$ & $19.62 \pm 7.19$ & $\boldsymbol{16.81 \pm 6.81}$ & $18.13 \pm 6.77$ & $18.77 \pm 6.76$ & $\underline{17.83 \pm 6.71}$ & $19.93 \pm 7.08$ & $24.57 \pm 7.34$ & $\textcolor{red}{14.67 \pm 5.38}$ \\
Chameleon & $11.50 \pm 2.51$ & $11.58 \pm 2.57$ & $11.24 \pm 2.78$ & $11.19 \pm 3.63$ & $\textcolor{red}{8.95 \pm 1.79}$ & $\underline{9.20 \pm 2.38}$ & $10.67 \pm 3.05$ & $11.00 \pm 3.43$ & $13.18 \pm 1.94$ & $\boldsymbol{9.06 \pm 2.35}$ \\
Squirrel & $6.19 \pm 2.43$ & $6.19 \pm 2.42$ & $6.17 \pm 2.53$ & $6.37 \pm 2.61$ & $\boldsymbol{4.47 \pm 1.99}$ & $\textcolor{red}{4.38 \pm 1.60}$ & $5.87 \pm 2.73$ & $7.31 \pm 3.04$ & $6.55 \pm 2.11$ & $\underline{4.97 \pm 1.53}$ \\
Actor & $2.47 \pm 0.98$ & $\underline{2.47 \pm 0.84}$ & $\textcolor{red}{1.96 \pm 0.76}$ & $\underline{2.47 \pm 0.93}$ & $2.68 \pm 1.01$ & $\boldsymbol{2.37 \pm 0.87}$ & $2.52 \pm 1.01$ & $4.38 \pm 2.13$ & $4.77 \pm 1.89$ & $2.49 \pm 0.85$ \\
Roman-Empire & $12.27 \pm 1.09$ & $\boldsymbol{2.64 \pm 0.65}$ & $\textcolor{red}{2.10 \pm 0.83}$ & $\underline{2.72 \pm 0.77}$ & $4.06 \pm 1.20$ & $2.80 \pm 0.84$ & $2.84 \pm 0.64$ & $3.37 \pm 0.95$ & $2.93 \pm 1.04$ & $4.36 \pm 0.78$ \\
tolokers & $4.30 \pm 1.09$ & $4.22 \pm 0.84$ & $3.66 \pm 0.81$ & $4.23 \pm 0.81$ & $\underline{2.92 \pm 0.95}$ & $\boldsymbol{2.67 \pm 0.81}$ & $3.75 \pm 0.93$ & $\textcolor{red}{2.04 \pm 0.63}$ & $4.25 \pm 0.65$ & $3.25 \pm 0.78$ \\
\midrule
\multicolumn{11}{l}{\textit{Large}} \\
ogbn-arxiv & $5.06 \pm 1.99$ & $2.65 \pm 0.52$ & $2.72 \pm 0.34$ & $2.57 \pm 0.55$ & $\textcolor{red}{1.29 \pm 0.18}$ & $\underline{1.82 \pm 0.36}$ & $2.63 \pm 0.77$ & $4.51 \pm 0.85$ & $2.19 \pm 0.58$ & $\boldsymbol{1.53 \pm 0.40}$ \\
Reddit & $5.35 \pm 1.59$ & $1.68 \pm 0.67$ & $1.84 \pm 0.58$ & $1.73 \pm 0.52$ & $\boldsymbol{1.14 \pm 0.13}$ & $\textcolor{red}{0.94 \pm 0.48}$ & -- & $1.92 \pm 1.31$ & $\underline{1.24 \pm 0.47}$ & $1.25 \pm 0.18$ \\
\midrule
Avg & $11.43$ & $6.04$ & $5.75$ & $5.62$ & $\boldsymbol{5.37}$ & $\underline{5.43}$ & $5.79$ & $6.32$ & $7.11$ & $\textcolor{red}{4.79}$ \\
Avg rank & -- & 5.61 & 5.17 & 4.58 & \textbf{3.94} & \underline{4.08} & 4.38 & 7.22 & 6.42 & \textcolor{red}{3.33} \\
\bottomrule
\end{tabular}%
}
\end{table}
\paragraph{Calibration Across Graph Regimes} \label{exp:main}
Table~\ref{tab:main_ece} reports test ECE. HoTS achieves the lowest mean ECE, $4.79\%$, and the best average rank, $3.33$, among the compared calibrators. On several homophilic datasets, the strongest baselines are statistically close. On PubMed, CoraFull, Roman-Empire, and tolokers, another method obtains the lowest ECE. HoTS improves the aggregate calibration profile while using a much smaller calibrator, and provides substantial improvements on benchmarks such as Texas, Cornell, and Wisconsin. This is consistent with the motivation of HoTS. A single global temperature can be sufficient when reliability errors are nearly uniform, whereas graph-dependent correction becomes more useful when local structure changes the meaning of confidence. Results under NLL and degree-stratified ECE are reported in Appendix~\ref{app:additional_metrics}.

\paragraph{Controlled Validation of the Structural Law} \label{exp:csbm} The theory predicts two qualitative effects. Theorem~\ref{thm:signal} implies that, in the positive-signal CSBM regime, the Bayes temperature decreases as the homophily signal $\Gamma(h)$ increases. Proposition~\ref{prop:alpha} predicts that noisy homophily estimates attenuate the apparent power-law exponent. Figure~\ref{fig:theory_validation} tests these two predictions in controlled CSBM experiments.
\begin{figure}[h!]
\centering
\includegraphics[width=0.80\textwidth]{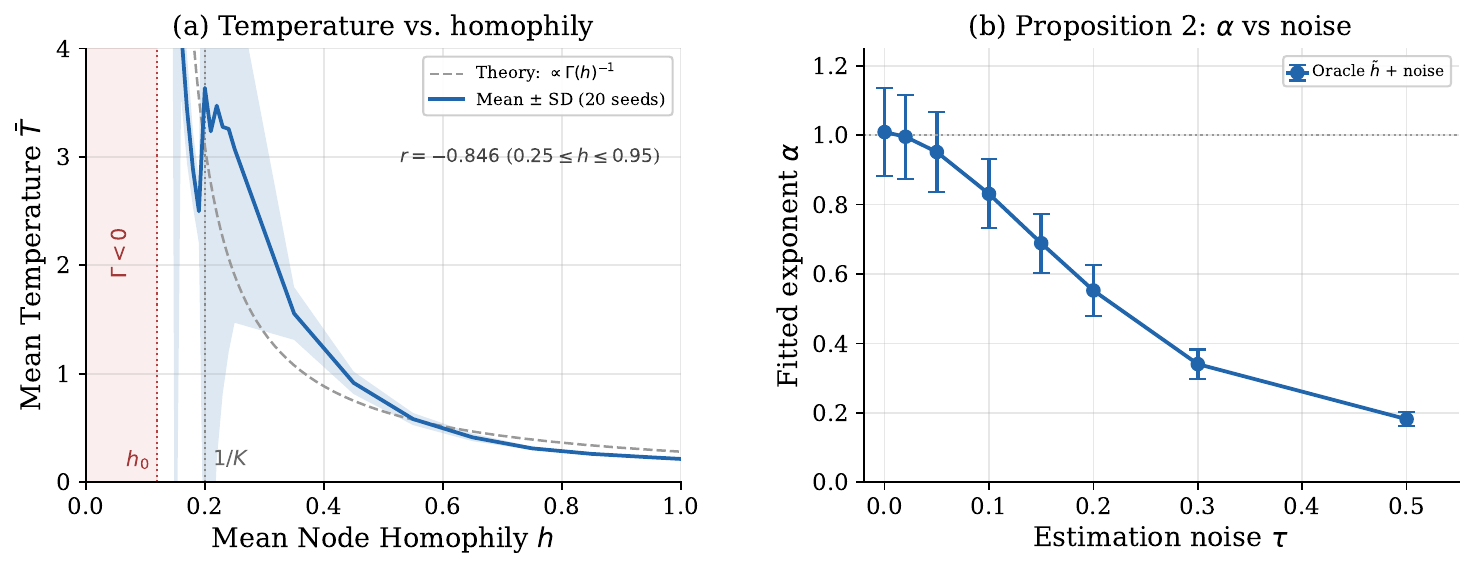}
\caption{Controlled CSBM validation. \textbf{(a)} Fitted temperature over the full homophily range; the blue curve and shaded band show the mean and one standard deviation across 20 seeds. The dashed curve shows the inverse-homophily trend. Pearson correlation is computed over the eight sweep points in $0.25\le h\le0.95$. We mark $h_0$ and $1/K$ and shade $\Gamma(h)<0$; the displayed vertical range is truncated. \textbf{(b)} The fitted exponent $\alpha$ is near the oracle value under noiseless homophily and decreases as estimation noise increases, matching Proposition~\ref{prop:alpha}'s attenuation.}
\label{fig:theory_validation}
\end{figure}

In the homophily sweep, all other CSBM parameters are fixed. Over the sweep range $0.25\le h\le0.95$, the mean fitted temperature decreases with homophily, with Pearson correlation $r=-0.846$ computed from the eight mean temperatures over 20 seeds, matching the predicted trend $T^*(h)\propto\Gamma(h)^{-1}$. Figure~\ref{fig:theory_validation}(a) also shows the additional experiments at lower homophily. In the noise sweep, we add controlled noise to the oracle homophily signal before fitting the power-law exponent. The fitted exponent decreases from $1.009$ at zero noise to $0.182$ at the largest noise level, consistent with the errors-in-variables attenuation formula. These experiments do not prove that real graphs follow the CSBM exactly. Instead, they verify that the structural dependencies used by HoTS are recoverable in the setting where the theory is meant to apply. On real benchmarks, fixing the exponent to the oracle value $\alpha=1$ increases mean ECE by $+1.10$ percentage points under the matched comparison in Appendix~\ref{app:component_ablation}, supporting the decision to learn $\alpha$ from calibration data.

\begin{table}[!htbp] 
\centering\footnotesize
\caption{Selective classification (top prediction-preserving baselines): retained-node accuracy (\%) at coverage $c$, macro-averaged   
over 18 datasets, 2 backbones, and 10 seeds. Restricted to baselines that strictly preserve the  
predicted class (see Section~\ref{app:vsgets}). Best per row in \textbf{bold} (excluding Uncal).}
\label{tab:selective}
\begin{tabular}{lcccc}
\toprule
Coverage & Uncal & TS & HTS & HoTS (ours) \\
\midrule
100\% & 67.39 & 67.39 & 67.39 & 67.39 \\
95\% & 68.87 & 68.90 & 68.92 & \textbf{68.94} \\
90\% & 70.01 & 70.08 & 70.13 & \textbf{70.15} \\
85\% & 71.04 & 71.12 & 71.17 & \textbf{71.18} \\
80\% & 71.94 & 72.05 & \textbf{72.13} & 72.11 \\
75\% & 72.81 & 72.92 & 72.94 & \textbf{72.96} \\
70\% & 73.60 & 73.73 & 73.75 & \textbf{73.81} \\
\bottomrule
\end{tabular}
\end{table}

\paragraph{Selective Classification} Calibration also guides when to trust a prediction. We therefore evaluate selective classification. Nodes are ranked by calibrated confidence, and the least confident nodes are rejected. Because HoTS uses positive scalar temperatures, it preserves the predicted class at full coverage. Any improvement below 100\% coverage comes from a better confidence ranking rather than from changing the classifier. HoTS gives the highest retained accuracy among the compared methods at most evaluated coverage levels (Table~\ref{tab:selective}). These results support the use of homophily-aware temperature scaling when downstream decisions depend on separating reliable predictions from less reliable ones.

\paragraph{Does Homophily Explain Residual Reliability?} \label{exp:homophily_split}
Proposition~\ref{prop:necessity} motivates examining whether local homophily is associated with reliability beyond the base-model confidence. We stratify test nodes by confidence and compare empirical accuracy across oracle-homophily groups within each bin. Using oracle homophily lets us examine the reliability information in neighborhood label composition itself, separately from errors in estimating it. The high-homophily group is more accurate on most datasets, with pronounced gaps on homophilic and large-scale graphs and mixed behavior on heterophilic graphs (Figure~\ref{fig:p1_homophily_split}). This retrospective diagnostic motivates homophily as a structural variable for calibration. True labels are used only to form the diagnostic groups; HoTS uses estimated homophily in its temperature map.

\paragraph{Additional Experiments} \label{exp:diagnostics_pointer}
Additional experiments are reported in the appendix to keep the main text focused. Additional GNN backbone comparisons are documented in Appendix~\ref{app:cross_backbone}. Appendix~\ref{app:additional_metrics} and Appendix~\ref{app:per_backbone} show robustness under NLL, degree-stratified ECE, and separate GCN and GAT evaluations. Appendix~\ref{app:reliability_diagnostics} shows that HoTS reduces the residual calibration gap after temperature scaling and relates node-level confidence gaps to entropy concentration and oracle homophily. Appendix~\ref{app:kl_validation} validates the sparse-overlap KL scaling used in the theory, and Appendix~\ref{app:shift} studies random edge perturbations. Appendix~\ref{app:homophily_diagnostics} extends the residual-reliability diagnostic of this section. Appendix~\ref{app:structural_comparison} compares homophily with representative motif, centrality, and curvature statistics after controlling for confidence and degree. Appendix~\ref{app:component_ablation} compares HoTS with entropy-only scaling to assess the contribution of homophily.

\section{Conclusion} \label{sec:Conclusion}

This paper presents a structural view of post-hoc calibration for graph neural networks. The main insight is that confidence cannot always be calibrated from logits alone, because the reliability of a graph prediction also depends on the local evidence induced by message passing. When nodes with comparable logits differ in neighborhood homophily, they can require different temperatures to produce reliable probabilities. We formalize this through a necessity result showing that a logit-only calibrator incurs an avoidable calibration gap whenever homophily explains residual variation in the optimal temperature. The CSBM analysis makes this principle explicit by showing that, for class-template scores under a population-concentration model, the Bayes inverse-temperature is governed by a homophily-dependent signal strength. In the positive-signal homophilic regime, this yields an approximate inverse relation between temperature and normalized local homophily.

HoTS translates this law into a three-parameter temperature map by combining entropy-based logit concentration, estimated homophily, and a learned exponent that adapts to noisy structural estimates. Since HoTS uses positive scalar temperatures, it improves probability reliability while preserving the predicted class. Across 18 benchmarks and both GCN and GAT backbones, HoTS achieves the best mean ECE and average rank among the compared calibrators. However, the exact Bayes characterization is derived under an idealized CSBM with Gaussian features, equidistant class means, uniform class priors, and a one-layer linear GCN in a population-concentration limit. Thus, the inverse-homophily law should be viewed as a structural prior that motivates HoTS rather than an exact Bayes rule for arbitrary real graphs, especially under class imbalance, nonlinear multi-layer message passing, or low-degree finite-sample fluctuations. Overall, the results show that a compact theory-guided structural correction can provide an interpretable and effective path toward reliable probabilistic graph learning.

\bibliographystyle{plainnat}
\bibliography{references}

\clearpage

\appendix
\clearpage
\onecolumn 
\thispagestyle{empty}

\begin{center}
\vspace*{0.1\textheight}

\hrule height 2pt
\vspace{0.8em}
{\LARGE \textbf{Supplementary Material}}
\vspace{0.8em}
\hrule height 1pt
\vspace{4em}

{\Large \textbf{Appendix Table of Contents}}
\vspace{1.5em}

\renewcommand{\arraystretch}{1.3}
\begin{tabular}{@{}p{0.2\textwidth} p{0.7\textwidth}@{}}
\toprule
\textbf{Appendix \ref{app:auxiliary}} & \textbf{Auxiliary Results and Proofs} \\
\midrule
& \ref{app:formal_csbm} Formal Population Concentration Model for Theorem~\ref{thm:signal} \\
& \ref{app:necessity_proof} Proof of Proposition~\ref{prop:necessity} \\
& \ref{app:bayes_proof} Proof of Theorem~\ref{thm:signal} \\
& \ref{app:sufficiency} Class-Symmetric Representative and Scalar $h$ Sufficiency \\
& \ref{app:entropy_expansions} Entropy and Logit Taylor Expansions \\
& \ref{app:binary_reliability} Binary CSBM Reliability Calculation \\
& \ref{app:alpha_proof} Proof of Proposition~\ref{prop:alpha} \\
& \ref{app:sparse_overlap} Sparse-Overlap Gaussian Residual Approximation \\
\midrule
\textbf{Appendix \ref{app:empirical}} & \textbf{Additional Empirical Results} \\
\midrule
& \ref{app:additional_metrics} Additional Calibration Metrics \\
& \ref{app:per_backbone} Per-Backbone Results \\
& \ref{app:cross_backbone} Calibration Across GNN Backbones \\
& \ref{app:reliability_diagnostics} Reliability Diagrams and Node-Level Diagnostics \\
& \ref{app:kl_validation} KL Bound Validation \\
& \ref{app:shift} Structural Perturbation Stress Test \\
& \ref{app:homophily_diagnostics} Confidence-Conditioned Homophily Diagnostics \\
& \ref{app:structural_comparison} Comparison with Other Structural Statistics \\
& \ref{app:component_ablation} HoTS Component Ablation \\
\midrule
\textbf{Appendix \ref{app:implementation}} & \textbf{Implementation Details} \\
\midrule
& \ref{app:datasets} Dataset Statistics and Preprocessing \\
& \ref{app:hardware} Hardware and Software \\
& \ref{app:gnn_hyperparams} GNN Training Hyperparameters \\
& \ref{app:cal_hyperparams} Calibration Method Hyperparameters \\
& \ref{app:hots_predictor} HoTS Homophily Predictor \\
& \ref{app:vsgets} Prediction Preservation Across Calibrators \\
\bottomrule
\end{tabular}

\vfill
\end{center}
\clearpage
\section{Auxiliary Results and Proofs}
\label{app:auxiliary}

This appendix collects the technical details omitted from the main text. The results are organized according to the order in which they are used. We first state the population-concentration CSBM used for the posterior calculation. We then prove the homophily necessity result, the CSBM posterior theorem, and the class-symmetric scalar-sufficiency statement. We next give the entropy expansion, the binary reliability calculation, and the exponent attenuation proof. The final theoretical subsection states the sparse-overlap approximation used to justify the neighbor-logit product approximation.

\subsection{Formal Population Concentration Model for Theorem~\ref{thm:signal}}
\label{app:formal_csbm}

The posterior calculation in Theorem~\ref{thm:signal} uses the following population-concentration version of the CSBM. The graph has $K$ classes with uniform prior. Edges are generated independently with probabilities $p_{\mathrm{in}}$ within classes and $p_{\mathrm{out}}$ across classes, with $p_{\mathrm{in}}\neq p_{\mathrm{out}}$. Node features satisfy
\begin{equation}
    x_j=\mu_{y_j}+\epsilon_j,\qquad
    \epsilon_j\sim\mathcal N(0,\sigma^2 I)
\end{equation}
independently, with $\sigma>0$.

The one layer aggregate is
\begin{equation}
    \bar x_i
    =
    \frac{1}{d_i+1}
    \left(x_i+\sum_{j\in\mathcal N(i)}x_j\right).
\end{equation}
For the posterior calculation, conditional on $Y_i=c$ and the realized oracle local homophily $H_i=h$, we replace the degree by its population value $d_i=\bar d$ and fix the neighbor class composition at its population-concentration value. A fraction $h$ of neighbors belongs to class $c$, and the remaining fraction $1-h$ is class balanced over the other $K-1$ classes. This deterministic composition approximation removes finite class composition fluctuations. The only remaining randomness in $\bar x_i$ comes from the Gaussian feature noises.

The class means are equidistant. There exists $c_\mu>0$ such that
\begin{equation}
    \|\mu_k-\mu_\ell\|_2=c_\mu,\qquad k\neq \ell.
\end{equation}
Let
\begin{equation}
    \bar\mu=\frac1K\sum_{r=1}^K\mu_r,\qquad
    m_c=\mu_c-\bar\mu,\qquad
    \rho_0=\|m_c\|.
\end{equation}
Then $\{m_c\}_{c=1}^K$ forms a regular simplex on its affine span. We define the class-template score vector by
\begin{equation}
    s_c=\rho_0^{-1}m_c^\top(\bar x_i-\bar\mu).
\end{equation}
Finally, write
\begin{equation}
    \tilde h=\frac{Kh-1}{K-1},\qquad
    v=\frac{\sigma^2}{\bar d+1},\qquad
    \Gamma(h)=\frac{\rho_0(1+\bar d\tilde h)}{\bar d+1}.
\end{equation}

\subsection{Proof of Proposition~\ref{prop:necessity}}
\label{app:necessity_proof}

For each $(z,h)$, let
\begin{equation}
    r_{z,h}(\tau)
    =
    \EE[\ell(\tau;Z,Y)\mid Z=z,H=h].
\end{equation}
Assume that $r_{z,h}$ has a unique interior minimizer $\tau^*(z,h)$. Let
\begin{equation}
    \kappa(z,h)
    =
    \left.
    \partial_\tau^2 r_{z,h}(\tau)
    \right|_{\tau=\tau^*(z,h)} .
\end{equation}
Define the curvature-weighted projection
\begin{equation}
    \bar\tau_\kappa(z)
    =
    \frac{
        \EE_H[\kappa(z,H)\tau^*(z,H)\mid Z=z]
    }{
        \EE_H[\kappa(z,H)\mid Z=z]
    } .
\end{equation}
Assume that the within-$z$ variation of $\tau^*(z,H)$ is locally small and that the third derivatives of $r_{z,h}$ are uniformly bounded in the relevant neighborhood. Let
\begin{equation}
    \epsilon_z
    =
    \sup_h|\tau^*(z,h)-\bar\tau_\kappa(z)|
\end{equation}
over the support of $H\mid Z=z$. Then
\begin{equation}
\begin{aligned}
    &
    \inf_{\tau\in\cT_z}\EE[\ell(\tau(Z);Z,Y)]
    -
    \inf_{\tau\in\cT_{z,h}}\EE[\ell(\tau(Z,H);Z,Y)]
    \\
    &\quad =
    \frac12\EE\!\left[
        \kappa(Z,H)
        \bigl(\tau^*(Z,H)-\bar\tau_\kappa(Z)\bigr)^2
    \right]
    +
    o\!\left(\EE_Z[\epsilon_Z^2]\right).
\end{aligned}
\end{equation}

For fixed $z$, the loss is convex in $\tau$ because
\begin{equation}
    \partial_\tau^2\ell(\tau;z,y)
    =
    \operatorname{Var}_{p_\tau}(z_k)
    \ge0,
\end{equation}
where $p_\tau=\mathrm{softmax}(\tau z)$. It is strictly convex whenever $z$ is not proportional to $\mathbf 1$.

The displayed expansion follows by Taylor expanding $r_{z,h}$ around $\tau^*(z, h)$ and then minimizing the leading quadratic term over a single logit-only value. The minimizer of this quadratic approximation is precisely $\bar{\tau}_\kappa(z)$. %
\textbf{Equivalently, $\bar{\tau}_\kappa(z)$ is the local-quadratic surrogate for the exact CE-optimal logit-only inverse-temperature $\tau^\star(z)$ in Proposition~\ref{prop:necessity}: under the bounded-third-derivative regularity above, $\tau^\star(z) = \bar{\tau}_\kappa(z) + o(\epsilon_z)$, so the displayed gap holds with $\tau^\star(z)$ in place of $\bar{\tau}_\kappa(z)$ at the same leading order.}
Hence the leading-order gap is a curvature-weighted conditional variance. If $\tau^*(z,H)$ is not almost surely constant given $Z=z$, strict convexity implies that no single logit-only value can attain all conditional minima simultaneously. This gives a strictly positive expected CE gap.

\subsection{Proof of Theorem~\ref{thm:signal}}
\label{app:bayes_proof}

Condition on $Y_i=c$ and $H_i=h$. Under the population-concentration model in Appendix~\ref{app:formal_csbm}, the self-loop contributes one copy of class $c$, the same class neighbors contribute $h\bar d$ population weighted copies of class $c$, and the off class neighbors are class-balanced over the other $K-1$ classes. Hence the conditional aggregate mean is
\begin{equation}
    \bar\mu+a(h)m_c,
    \qquad
    a(h)=\frac{1+\bar d\tilde h}{\bar d+1}.
\end{equation}
Since the class composition is fixed at its population value, the only remaining randomness is the independent Gaussian feature noise. The aggregate covariance is $vI$, with $v=\sigma^2/(\bar d+1)$. Thus
\begin{equation}
    \bar x_i\mid(Y_i=c,H_i=h)
    \sim
    \cN(\bar\mu+a(h)m_c,vI).
\end{equation}

Decompose $\bar x_i-\bar\mu$ into its projection onto $\operatorname{span}\{m_1,\ldots,m_K\}$ and its orthogonal component. The orthogonal component has a class independent Gaussian distribution. The likelihood ratio depends on $\bar x_i$ only through
\begin{equation}
    \left(
    m_1^\top(\bar x_i-\bar\mu),
    \ldots,
    m_K^\top(\bar x_i-\bar\mu)
    \right),
\end{equation}
or equivalently through the class-template score vector $s$. Hence the posterior conditional on the full aggregate is the same function of $s$ as the posterior conditional on $s$ alone. Since the class prior is uniform and the model is class-symmetric, $P(Y_i=c\mid H_i=h)=1/K$. Bayes rule gives a posterior proportional to
\begin{equation}
    \exp\!\left(
        -\frac{\|\bar x_i-\bar\mu-a(h)m_c\|^2}{2v}
    \right).
\end{equation}

Expanding the squared norm, the terms $\|\bar x_i-\bar\mu\|^2$ and $a(h)^2\|m_c\|^2$ are class independent and cancel in the softmax. The remaining term is
\begin{equation}
    \frac{a(h)}{v}(\bar x_i-\bar\mu)^\top m_c .
\end{equation}
Since $s_c=\rho_0^{-1}m_c^\top(\bar x_i-\bar\mu)$ and $\Gamma(h)=a(h)\rho_0$, we obtain
\begin{equation}
    P(Y_i=c\mid s,H_i=h)
    =
    \mathrm{softmax}\!\left(
        \frac{\Gamma(h)}{v}s
    \right)_c .
\end{equation}

Finally, if $z$ has a unique maximum coordinate $\hat y$, then $\mathrm{softmax}(\tau z)_{\hat y}$ is strictly increasing in $\tau$, because its derivative is
\begin{equation}
    p_{\hat y}
    \left(
        z_{\hat y}-\sum_k p_k z_k
    \right)
    >
    0 .
\end{equation}
Since $\Gamma(h)$ is strictly increasing in $h$, the posterior correctness probability is strictly increasing in $h$ on
\begin{equation}
    \mathcal H_+
    =
    \{h\mid \Gamma(h)>0\}.
\end{equation}

\paragraph{Strictness of the homophily gap}
Restricted to the positive-signal subpopulation $H\in\mathcal H_+$, the CSBM posterior verifies the strictness condition in Proposition~\ref{prop:necessity}. If the conditional distribution of $H$ on this subpopulation has at least two positive probability values in $\mathcal H_+$, then the Gaussian class-template score densities induced by those homophily values have overlapping support on the centered score subspace. Hence $P(H=h\mid S=s,H\in\mathcal H_+)$ is non-degenerate for almost every non-degenerate $s$. Since $\tau_{\mathrm{Bayes}}(h)=\Gamma(h)/v$ is strictly increasing in $h$ on $\mathcal H_+$, the conditional variance of $\tau_{\mathrm{Bayes}}(H)$ given $S=s$ is positive for almost every such $s$. Therefore, under the population-concentration CSBM posterior, the positive temperature calibration gap in Proposition~\ref{prop:necessity} is strict whenever homophily varies within fixed class-template score groups on $\mathcal H_+$.

Assume that $H$ has at least two positive probability values $h\neq h'$ in $\mathcal H_+$. For each such value, the conditional distribution of the centered aggregate is a non-degenerate Gaussian in feature space because $\sigma>0$. After projection onto the centered class-template logit subspace, the induced conditional density of $Z$ is strictly positive on that subspace. Therefore the densities $f_{Z\mid H=h}(z)$ and $f_{Z\mid H=h'}(z)$ overlap on a set of full measure in the centered logit subspace.

By Bayes rule, both $P(H=h\mid Z=z)$ and $P(H=h'\mid Z=z)$ are positive for almost every such $z$. Hence $H$ is not a deterministic function of $Z$. Since $\tau_{\mathrm{Bayes}}(h)=\Gamma(h)/v$ is strictly increasing in $h$ on $\mathcal H_+$, the conditional distribution of $\tau_{\mathrm{Bayes}}(H)$ given $Z=z$ is non-degenerate for almost every non-degenerate $z$. This proves the strictness claim used after Theorem~\ref{thm:signal}.

For non-atomic or binned homophily distributions, the same proof applies after replacing the two atoms by two disjoint measurable subsets of $\mathcal H_+$ with positive probability and overlapping conditional logit densities.

\paragraph{Positive and negative signal}
At $\htilde=0$, corresponding to a random neighborhood, $\Gamma(h)=\rho_0/(\dbar+1)>0$. Thus the node's own feature contributes signal even when the neighborhood is uninformative. The approximation $\Gamma(h)\approx\Gamma_0\htilde$ is valid only in the large-degree homophilic regime $\dbar\htilde\gg1$. If $1+\dbar\htilde<0$, then $\Gamma(h)<0$ and the signed Bayes inverse-temperature is negative. A negative inverse-temperature reverses the template logit ordering and cannot be represented by ordinary positive temperature scaling. Therefore exact positive temperature Bayes optimality statements are restricted to $\mathcal H_+$. The implemented HoTS method uses an absolute value extension for stability outside this regime. This extension should be interpreted as a calibration parametrization rather than as a signed Bayes posterior approximation.

\paragraph{Scope of the scalar temperature}
Under the population-concentration CSBM with equidistant class means, the centered class-symmetric representative makes the scalar group level Bayes temperature depend on the graph only through $h$. This is a statement about the idealized group level scalar temperature. It does not imply that arbitrary graph-aware features, such as neighbor features, neighbor logits, or finite-degree composition fluctuations, are conditionally irrelevant in finite CSBM or in real graphs. The symmetry argument is provided in Appendix~\ref{app:sufficiency}, and the sparse-overlap neighbor logit approximation is given in Appendix~\ref{app:sparse_overlap}.

\subsection{Class-Symmetric Representative and Scalar $h$ Sufficiency}
\label{app:sufficiency}

This subsection records the symmetry argument behind the statement that, in the population-concentration CSBM, the scalar group level temperature depends on the graph only through $h$.

Let $m_c=\mu_c-\bar\mu$. Under the equidistant means assumption in Appendix~\ref{app:formal_csbm}, the set $\{m_c\}_{c=1}^K$ is a regular simplex on its span. For every class permutation $\pi\in S_K$, there is an orthogonal map $R_\pi$ on the simplex span such that $R_\pi m_c=m_{\pi(c)}$. Let $P_\pi$ be the corresponding permutation matrix on output coordinates. Because the theoretical head acts on centered features, the CSBM distribution is invariant under the simultaneous transformation
\begin{equation}
    (y,x-\bar\mu)
    \mapsto
    (\pi(y),R_\pi(x-\bar\mu)).
\end{equation}
Hence the population CE risk is invariant under
\begin{equation}
    W
    \mapsto
    P_\pi W R_\pi^{-1}.
\end{equation}
The population CE risk is convex in the linear classifier $W$. Therefore, if $W^*$ is a population minimizer, the symmetrized classifier
\begin{equation}
    W^\dagger
    =
    |S_K|^{-1}
    \sum_{\pi\in S_K}P_\pi W^*R_\pi^{-1}
\end{equation}
is also a minimizer.

The symmetrized classifier is class equivariant, meaning that
\begin{equation}
    W^\dagger R_\pi=P_\pi W^\dagger .
\end{equation}
On the centered simplex span, this implies a one versus rest symmetric form, up to a softmax irrelevant constant shift of all logits. Under degree and off class composition concentration, the conditional distribution of logits given $(H_i=h,Y_i=c)$ depends on $c$ only through a permutation of coordinates. Therefore the scalar group level CE objective within a fixed $h$ group depends on the graph only through $h$, and the corresponding group optimal positive temperature is a function of $h$ alone on $\mathcal H_+$.

This is an idealized population-concentration statement. It does not say that arbitrary graph-aware information is redundant in finite CSBM or in real graphs, where degree variation, off class composition noise, learned logits, or neighbor features may carry additional calibration information.

\subsection{Entropy and Logit Taylor Expansions}
\label{app:entropy_expansions}

For the binary expansion, let $\Delta=z_1-z_2$. The binary predictive probability is $\sigma(\Delta)$, and the entropy is $H(\sigma(\Delta))$. A Taylor expansion around $\Delta=0$ gives
\begin{equation}
    H(\sigma(\Delta))
    =
    \log2-\frac{\Delta^2}{8}+O(\Delta^4).
\end{equation}
After normalizing by $\log2$ and inverting,
\begin{equation}
    |\Delta|
    =
    \sqrt{8\log2(1-e)}
    +
    O((1-e)^{3/2}).
\end{equation}

For the multi-class expansion, use softmax shift invariance and assume $\sum_k z_k=0$. Around $z=0$, the entropy satisfies
\begin{equation}
    H(\mathrm{softmax}(z))
    =
    \log K-\frac{\|z\|^2}{2K}+O(\|z\|^3).
\end{equation}
The quadratic term is the Fisher information at the uniform distribution restricted to the centered subspace. Dividing by $\log K$ gives
\begin{equation}
    1-e
    =
    \frac{\|z\|^2}{2K\log K}
    +
    O(\|z\|^3).
\end{equation}
Inversion yields
\begin{equation}
    \|z^\circ\|
    =
    \sqrt{2K\log K(1-e)}
    +
    O(1-e).
\end{equation}
For $K>2$, the cubic term generally does not vanish, so the remainder is $O(1-e)$ after taking the square root.

These expansions are local around the uniform predictive distribution. In HoTS, $\sqrt{2K\log K(1-e_i)}$ is used as a simple monotone proxy for logit concentration over the full entropy range rather than as a global identity.

\subsection{Binary CSBM Reliability Calculation}
\label{app:binary_reliability}

For $K=2$, the centered simplex templates satisfy $m_2=-m_1$. Under the centered class-template scores,
\begin{equation}
    s_1=\rho_0^{-1}m_1^\top(\bar x-\bar\mu),
    \qquad
    s_2=-s_1 .
\end{equation}
Conditional on $Y=1$ and $H=h$, the population-concentration model gives
\begin{equation}
    s_1\sim\cN(\Gamma(h),v).
\end{equation}
The class-template prediction is correct when $s_1>s_2$, equivalently when $s_1>0$. Hence
\begin{equation}
    q(h)
    =
    P(s_1>0\mid Y=1,H=h)
    =
    \Phi(\Gamma(h)/\sqrt v).
\end{equation}
By symmetry, the same class-template correctness probability holds conditional on $Y=2$.

Thus the exact logit reliability is
\begin{equation}
    \operatorname{logit}(q(h))
    =
    \operatorname{logit}\Phi(\Gamma(h)/\sqrt v).
\end{equation}
The approximation
\begin{equation}
    \operatorname{logit}(\Phi(x))
    \approx
    \frac{\pi}{\sqrt3}x
\end{equation}
is the standard variance matching logistic probit approximation. Combining it with the large $\dbar\htilde$ relation $\Gamma(h)\approx\Gamma_0\htilde$ gives
\begin{equation}
    \operatorname{logit}(q(h))
    \approx
    \frac{\pi\Gamma_0}{\sqrt{3v}}\htilde .
\end{equation}
This linear relation is a moderate-margin approximation. It is not a large-margin asymptotic, since for very large positive $x$, $\operatorname{logit}(\Phi(x))$ is not exactly linear.

\subsection{Proof of Proposition~\ref{prop:alpha}}
\label{app:alpha_proof}

The population regression slope of $\log T^*$ on $\hat u$ is
\begin{equation}
    \frac{\operatorname{Cov}(\hat u,\log T^*)}
    {\operatorname{Var}(\hat u)}.
\end{equation}
Since $\log T^*=a-u$ and $\hat u=u+\eta$, with $\eta$ independent of $u$ and mean zero, we have
\begin{equation}
    \operatorname{Cov}(\hat u,\log T^*)
    =
    -\operatorname{Var}(u)
\end{equation}
and
\begin{equation}
    \operatorname{Var}(\hat u)
    =
    \operatorname{Var}(u)+\tau^2 .
\end{equation}
Thus the fitted slope is
\begin{equation}
    -
    \frac{\operatorname{Var}(u)}
    {\operatorname{Var}(u)+\tau^2}.
\end{equation}
Because the model writes the slope as $-\alpha$, this gives
\begin{equation}
    \alpha^*
    =
    \frac{\operatorname{Var}(u)}
    {\operatorname{Var}(u)+\tau^2}.
\end{equation}
The strictness condition follows immediately.

\subsection{Sparse-Overlap Gaussian Residual Approximation}
\label{app:sparse_overlap}

This subsection records the technical approximation for neighbor-logits. It is used only to justify that, in a sparse-overlap regime, the Gaussian residual dependence among neighbor-logits is small after conditioning on the center feature $x_i$.

Condition on $x_i$. For two distinct neighbors $j,j'\in\cN(i)$, define $S_{jj'}$ as the number of shared residual aggregation inputs in
\begin{equation}
    (\cN(j)\cup\{j\})
    \cap
    (\cN(j')\cup\{j'\})
\end{equation}
after excluding the common center node $i$. Suppose
\begin{equation}
    \EE[S_{jj'}]=O(K\dbar^2/n),
    \qquad
    \EE[S_{jj'}^2]=O(K\dbar^2/n).
\end{equation}
Work on the effective non-degenerate residual logit subspace, equivalently using the Moore-Penrose inverse for whitening.

Let $W$ be the fixed linear head. Conditional on $x_i$, the residual covariance of one neighbor-logit vector is proportional to
\begin{equation}
    \frac{\dbar\sigma^2WW^\top}{(\dbar+1)^2},
\end{equation}
while the cross covariance between two neighbor residuals is proportional to
\begin{equation}
    \frac{S_{jj'}\sigma^2WW^\top}{(\dbar+1)^2}.
\end{equation}
After whitening on the effective rank $r$ logit subspace, the off-diagonal residual correlation block is
\begin{equation}
    R_{jj'}
    =
    \frac{S_{jj'}}{\dbar}I_r,
    \qquad
    r\leq K.
\end{equation}
Thus
\begin{equation}
    \|R_{jj'}\|_F^2
    \leq
    \frac{K S_{jj'}^2}{\dbar^2}.
\end{equation}
By the sparse-overlap assumption, its expectation is $O(K^2/n)$.

Let $\mathcal R$ be the block off-diagonal whitened residual correlation matrix for all neighbor residual logits. The joint whitened Gaussian residual covariance is $I+\mathcal R$, whereas the product of marginals has covariance $I$. The Gaussian KL from the joint residual law to the product law is
\begin{equation}
    -\frac12\log\det(I+\mathcal R),
\end{equation}
because $\mathcal R$ has zero diagonal blocks. If $\|\mathcal R\|_{\mathrm{op}}=o_p(1)$ and the log-determinant expansion is uniformly integrable, then the KL is
\begin{equation}
    \frac14\|\mathcal R\|_F^2
    +
    o(\|\mathcal R\|_F^2)
\end{equation}
in expectation. Since there are $O(\dbar^2)$ neighbor pairs and each contributes $O(K^2/n)$ in expectation, the expected Gaussian residual KL is
\begin{equation}
    O(K^2\dbar^2/n).
\end{equation}

Therefore, in the sparse-overlap regime $K^2\dbar^2/n\to0$, the product approximation for neighbor-logits is leading-order accurate for the Gaussian residual dependence. This is not an exact full finite CSBM independence theorem. Shared class labels, finite degree composition randomness, and learned logit effects can introduce additional dependence terms unless separately controlled.

\clearpage
\section{Additional Empirical Results}
\label{app:empirical}

This appendix provides additional empirical evidence supporting the main claims.
Appendix~\ref{app:additional_metrics} reports complementary calibration metrics beyond ECE.
Appendix~\ref{app:per_backbone} breaks down the results by GCN and GAT backbones.
Appendix~\ref{app:cross_backbone} records the additional backbone comparison and its experimental protocol.
Appendix~\ref{app:reliability_diagnostics} visualizes reliability diagrams and node-level confidence changes.
Appendix~\ref{app:kl_validation} checks the sparse-overlap KL scaling used in the theoretical approximation.
Appendix~\ref{app:shift} evaluates calibration under random structural perturbations.
Appendix~\ref{app:homophily_diagnostics} examines whether oracle local homophily is associated with prediction correctness after conditioning on confidence.
Appendix~\ref{app:structural_comparison} compares the residual reliability signal of oracle homophily with other structural statistics.
Appendix~\ref{app:component_ablation} compares HoTS with entropy-only scaling and examines learning versus fixing the homophily exponent.

\subsection{Additional Calibration Metrics}
\label{app:additional_metrics}
\label{app:metrics}

ECE is the primary metric because it directly measures confidence and accuracy mismatch, but it is a binned statistic. Table~\ref{tab:nll_degece} therefore reports two complementary summaries, mean NLL and degree-stratified ECE. The degree-stratified metric evaluates calibration after grouping nodes by degree, which is important for graph data because message passing reliability can vary with local connectivity.

\begin{table}[h]
\centering\small
\caption{Pooled (GCN $+$ GAT) mean NLL and degree-stratified ECE (\%) across 18 datasets and 10 seeds per backbone.}
\label{tab:nll_degece}
\begin{tabular}{lcccccccc}
\toprule
Metric & TS & ETS & HTS & CaGCN & GATS & GETS & WATS & HoTS (ours) \\
\midrule
NLL & 0.97 & 0.94 & 0.94 & 0.97 & 0.98 & 0.95 & 1.02 & \textbf{0.92} \\
deg\_ECE & 6.50 & 6.32 & 6.13 & \textbf{5.81} & 6.45 & 6.16 & 7.16 & 5.90 \\
\bottomrule
\end{tabular}
\end{table}

\clearpage
\subsection{Per-Backbone Results}
\label{app:per_backbone}

The main table averages over GCN and GAT to summarize the benchmark level behavior. Tables~\ref{tab:app_ece}, \ref{tab:app_deg_ece}, and \ref{tab:app_nll} give the corresponding per-backbone results. These tables are useful for diagnosing where the aggregate trends come from. HoTS is not uniformly best for every backbone and dataset, but its average behavior remains strong across the two architectures.

\begin{table}[!h]
\centering
\setlength{\tabcolsep}{2pt}
\caption{Per-backbone ECE (\%, mean$\pm$std over 10 seeds per backbone). Best per (dataset, backbone) row in \textcolor{red}{red}, second in \textbf{bold}, third \underline{underlined} (excluding Uncal).}
\label{tab:app_ece}
\resizebox{\textwidth}{!}{%
\begin{tabular}{llcccccccccc}
\toprule
Dataset & Backbone & Uncal & TS & VS & ETS & HTS & CaGCN & GATS & GETS & WATS & HoTS \\
\midrule
\multicolumn{12}{l}{\textit{Homophilic}} \\
\multirow{2}{*}{Cora} & GCN & $21.49 \pm 1.15$ & $3.42 \pm 0.70$ & $\boldsymbol{3.05 \pm 0.97}$ & $\underline{3.19 \pm 0.77}$ & $3.66 \pm 0.45$ & $4.08 \pm 0.88$ & $3.30 \pm 0.96$ & $4.16 \pm 0.92$ & $3.20 \pm 0.69$ & $\textcolor{red}{2.28 \pm 0.76}$ \\
 & GAT & $17.79 \pm 2.37$ & $2.86 \pm 0.83$ & $2.75 \pm 0.48$ & $\boldsymbol{2.69 \pm 0.78}$ & $3.49 \pm 0.64$ & $4.17 \pm 0.67$ & $\textcolor{red}{2.68 \pm 0.85}$ & $3.55 \pm 0.56$ & $2.95 \pm 0.75$ & $\underline{2.72 \pm 0.59}$ \\
\multirow{2}{*}{CiteSeer} & GCN & $21.87 \pm 3.46$ & $\boldsymbol{2.91 \pm 0.42}$ & $3.16 \pm 0.73$ & $\textcolor{red}{2.90 \pm 0.44}$ & $3.68 \pm 0.98$ & $4.93 \pm 1.08$ & $3.42 \pm 0.77$ & $5.66 \pm 1.10$ & $\underline{3.01 \pm 0.47}$ & $3.36 \pm 1.02$ \\
 & GAT & $15.26 \pm 1.17$ & $\underline{3.17 \pm 0.74}$ & $3.24 \pm 0.32$ & $\textcolor{red}{2.96 \pm 0.78}$ & $3.88 \pm 0.71$ & $4.69 \pm 0.22$ & $\textcolor{red}{2.96 \pm 0.87}$ & $4.49 \pm 0.63$ & $3.18 \pm 0.65$ & $3.65 \pm 0.44$ \\
\multirow{2}{*}{PubMed} & GCN & $13.90 \pm 1.58$ & $1.18 \pm 0.20$ & $1.44 \pm 0.31$ & $1.18 \pm 0.23$ & $1.45 \pm 0.65$ & $1.28 \pm 1.31$ & $\boldsymbol{0.91 \pm 0.30}$ & $\underline{1.14 \pm 0.38}$ & $\textcolor{red}{0.74 \pm 0.24}$ & $1.25 \pm 0.37$ \\
 & GAT & $10.41 \pm 0.83$ & $0.89 \pm 0.28$ & $\underline{0.87 \pm 0.23}$ & $\textcolor{red}{0.83 \pm 0.28}$ & $1.00 \pm 0.57$ & $0.94 \pm 0.41$ & $\boldsymbol{0.85 \pm 0.25}$ & $1.13 \pm 0.39$ & $0.92 \pm 0.33$ & $1.27 \pm 0.22$ \\
\multirow{2}{*}{CoraFull} & GCN & $28.50 \pm 0.92$ & $3.40 \pm 0.44$ & $3.22 \pm 0.58$ & $\underline{2.71 \pm 0.39}$ & $3.94 \pm 0.52$ & $2.87 \pm 0.54$ & $3.38 \pm 0.41$ & $3.74 \pm 0.84$ & $\textcolor{red}{2.43 \pm 0.62}$ & $\boldsymbol{2.47 \pm 0.55}$ \\
 & GAT & $37.35 \pm 0.88$ & $\boldsymbol{2.11 \pm 0.52}$ & $2.47 \pm 0.42$ & $\textcolor{red}{1.68 \pm 0.31}$ & $3.94 \pm 0.73$ & $2.26 \pm 0.52$ & $\underline{2.17 \pm 0.61}$ & $3.71 \pm 1.10$ & $5.27 \pm 1.01$ & $2.99 \pm 0.93$ \\
\multirow{2}{*}{Computers} & GCN & $5.12 \pm 0.81$ & $2.33 \pm 0.48$ & $2.27 \pm 0.35$ & $2.26 \pm 0.58$ & $\underline{1.51 \pm 0.35}$ & $\textcolor{red}{1.20 \pm 0.24}$ & $1.80 \pm 0.59$ & $2.62 \pm 0.58$ & $\boldsymbol{1.23 \pm 0.35}$ & $1.62 \pm 0.40$ \\
 & GAT & $6.60 \pm 1.52$ & $2.04 \pm 0.37$ & $\boldsymbol{1.87 \pm 0.45}$ & $1.98 \pm 0.36$ & $2.06 \pm 0.57$ & $\textcolor{red}{1.57 \pm 0.33}$ & $2.01 \pm 0.33$ & $2.70 \pm 0.86$ & $2.01 \pm 0.53$ & $\underline{1.90 \pm 0.52}$ \\
\multirow{2}{*}{Photo} & GCN & $3.68 \pm 0.62$ & $1.24 \pm 0.32$ & $1.30 \pm 0.23$ & $1.24 \pm 0.34$ & $\textcolor{red}{0.96 \pm 0.37}$ & $1.29 \pm 0.47$ & $\underline{1.17 \pm 0.25}$ & $1.69 \pm 0.35$ & $1.18 \pm 0.41$ & $\boldsymbol{0.98 \pm 0.39}$ \\
 & GAT & $4.93 \pm 0.69$ & $\boldsymbol{1.22 \pm 0.30}$ & $\underline{1.24 \pm 0.27}$ & $1.37 \pm 0.59$ & $1.31 \pm 0.52$ & $1.41 \pm 0.36$ & $1.32 \pm 0.37$ & $\textcolor{red}{1.16 \pm 0.29}$ & $1.28 \pm 0.36$ & $1.44 \pm 0.34$ \\
\multirow{2}{*}{CS} & GCN & $1.58 \pm 0.19$ & $1.57 \pm 0.18$ & $\underline{1.44 \pm 0.18}$ & $1.49 \pm 0.18$ & $\boldsymbol{1.37 \pm 0.36}$ & $2.79 \pm 0.83$ & $1.57 \pm 0.22$ & $1.59 \pm 0.16$ & $3.53 \pm 0.43$ & $\textcolor{red}{1.08 \pm 0.22}$ \\
 & GAT & $3.86 \pm 0.51$ & $2.95 \pm 0.64$ & $2.62 \pm 0.46$ & $2.95 \pm 0.65$ & $\textcolor{red}{1.32 \pm 0.26}$ & $\underline{1.99 \pm 0.61}$ & $2.66 \pm 0.51$ & $3.69 \pm 3.70$ & $2.98 \pm 0.60$ & $\boldsymbol{1.94 \pm 0.55}$ \\
\multirow{2}{*}{Physics} & GCN & $0.65 \pm 0.16$ & $0.62 \pm 0.14$ & $0.64 \pm 0.15$ & $0.55 \pm 0.11$ & $\textcolor{red}{0.39 \pm 0.13}$ & $0.80 \pm 0.32$ & $\underline{0.47 \pm 0.11}$ & $0.64 \pm 0.24$ & $1.80 \pm 0.28$ & $\boldsymbol{0.41 \pm 0.13}$ \\
 & GAT & $1.75 \pm 0.38$ & $1.12 \pm 0.15$ & $1.05 \pm 0.14$ & $1.09 \pm 0.19$ & $\boldsymbol{0.66 \pm 0.16}$ & $\textcolor{red}{0.64 \pm 0.14}$ & $0.97 \pm 0.19$ & $1.37 \pm 0.71$ & $1.03 \pm 0.25$ & $\underline{0.67 \pm 0.23}$ \\
\midrule
\multicolumn{12}{l}{\textit{Heterophilic}} \\
\multirow{2}{*}{Texas} & GCN & $25.38 \pm 6.99$ & $25.08 \pm 7.10$ & $\underline{19.96 \pm 7.05}$ & $24.16 \pm 6.21$ & $\boldsymbol{19.30 \pm 7.09}$ & $23.18 \pm 6.36$ & $21.83 \pm 7.24$ & $20.14 \pm 4.94$ & $27.84 \pm 8.25$ & $\textcolor{red}{16.67 \pm 4.49}$ \\
 & GAT & $18.35 \pm 7.91$ & $18.19 \pm 8.36$ & $16.50 \pm 6.35$ & $\underline{14.38 \pm 7.43}$ & $16.08 \pm 5.88$ & $17.47 \pm 8.10$ & $\textcolor{red}{13.87 \pm 4.94}$ & $16.43 \pm 4.28$ & $21.72 \pm 6.86$ & $\boldsymbol{14.08 \pm 6.79}$ \\
\multirow{2}{*}{Cornell} & GCN & $22.49 \pm 9.37$ & $21.70 \pm 8.67$ & $22.18 \pm 5.77$ & $\underline{19.19 \pm 8.74}$ & $19.58 \pm 7.83$ & $\boldsymbol{16.61 \pm 6.67}$ & $20.04 \pm 8.45$ & $21.97 \pm 7.92$ & $24.96 \pm 7.76$ & $\textcolor{red}{13.31 \pm 5.35}$ \\
 & GAT & $17.76 \pm 5.75$ & $18.40 \pm 5.37$ & $17.11 \pm 6.36$ & $17.16 \pm 6.08$ & $\underline{16.53 \pm 6.57}$ & $\boldsymbol{15.26 \pm 7.26}$ & $17.11 \pm 7.03$ & $17.19 \pm 7.12$ & $25.38 \pm 5.83$ & $\textcolor{red}{15.24 \pm 6.11}$ \\
\multirow{2}{*}{Wisconsin} & GCN & $21.44 \pm 7.41$ & $21.40 \pm 6.80$ & $22.26 \pm 6.65$ & $\boldsymbol{18.03 \pm 6.18}$ & $20.66 \pm 6.40$ & $\underline{20.43 \pm 6.08}$ & $20.69 \pm 5.83$ & $22.12 \pm 6.92$ & $26.29 \pm 7.65$ & $\textcolor{red}{16.59 \pm 5.09}$ \\
 & GAT & $17.18 \pm 6.55$ & $16.72 \pm 6.79$ & $16.97 \pm 6.72$ & $\underline{15.60 \pm 7.18}$ & $\underline{15.60 \pm 6.16}$ & $17.11 \pm 6.99$ & $\boldsymbol{14.96 \pm 6.29}$ & $17.75 \pm 6.55$ & $22.85 \pm 6.59$ & $\textcolor{red}{12.75 \pm 4.96}$ \\
\multirow{2}{*}{Chameleon} & GCN & $12.73 \pm 2.15$ & $12.81 \pm 2.18$ & $12.05 \pm 2.81$ & $13.03 \pm 1.92$ & $\textcolor{red}{9.34 \pm 2.08}$ & $\underline{10.39 \pm 2.03}$ & $12.25 \pm 2.08$ & $13.28 \pm 2.41$ & $12.80 \pm 2.25$ & $\boldsymbol{9.44 \pm 2.76}$ \\
 & GAT & $10.27 \pm 2.23$ & $10.34 \pm 2.31$ & $10.43 \pm 2.50$ & $9.36 \pm 4.00$ & $\boldsymbol{8.56 \pm 1.34}$ & $\textcolor{red}{8.01 \pm 2.08}$ & $9.08 \pm 3.05$ & $8.72 \pm 2.71$ & $13.55 \pm 1.47$ & $\underline{8.67 \pm 1.78}$ \\
\multirow{2}{*}{Squirrel} & GCN & $8.09 \pm 1.53$ & $8.10 \pm 1.45$ & $8.19 \pm 1.58$ & $8.57 \pm 1.14$ & $\textcolor{red}{5.13 \pm 1.98}$ & $\boldsymbol{5.14 \pm 1.55}$ & $8.22 \pm 1.53$ & $10.11 \pm 1.33$ & $7.48 \pm 1.44$ & $\underline{5.69 \pm 1.22}$ \\
 & GAT & $4.29 \pm 1.49$ & $4.29 \pm 1.53$ & $4.15 \pm 1.48$ & $4.16 \pm 1.61$ & $\underline{3.81 \pm 1.78}$ & $\boldsymbol{3.62 \pm 1.25}$ & $\textcolor{red}{3.52 \pm 1.26}$ & $4.50 \pm 0.98$ & $5.61 \pm 2.26$ & $4.24 \pm 1.46$ \\
\multirow{2}{*}{Actor} & GCN & $2.23 \pm 0.85$ & $2.21 \pm 0.87$ & $\textcolor{red}{1.88 \pm 0.60}$ & $\boldsymbol{2.11 \pm 0.93}$ & $2.33 \pm 1.04$ & $2.16 \pm 0.94$ & $\boldsymbol{2.11 \pm 1.06}$ & $2.97 \pm 1.53$ & $5.37 \pm 2.17$ & $2.33 \pm 1.06$ \\
 & GAT & $2.71 \pm 1.04$ & $2.73 \pm 0.73$ & $\textcolor{red}{2.04 \pm 0.88}$ & $2.82 \pm 0.77$ & $3.04 \pm 0.84$ & $\boldsymbol{2.57 \pm 0.74}$ & $2.92 \pm 0.76$ & $5.79 \pm 1.65$ & $4.18 \pm 1.33$ & $\underline{2.65 \pm 0.53}$ \\
\multirow{2}{*}{Roman-Empire} & GCN & $12.44 \pm 0.80$ & $\underline{2.39 \pm 0.55}$ & $\textcolor{red}{1.57 \pm 0.53}$ & $\boldsymbol{2.32 \pm 0.58}$ & $5.02 \pm 0.65$ & $2.42 \pm 0.55$ & $3.05 \pm 0.58$ & $2.76 \pm 0.60$ & $3.21 \pm 1.23$ & $4.45 \pm 0.59$ \\
 & GAT & $12.10 \pm 1.29$ & $2.89 \pm 0.66$ & $\boldsymbol{2.64 \pm 0.73}$ & $3.12 \pm 0.73$ & $3.10 \pm 0.80$ & $3.17 \pm 0.91$ & $\textcolor{red}{2.63 \pm 0.64}$ & $3.99 \pm 0.82$ & $\underline{2.66 \pm 0.71}$ & $4.26 \pm 0.93$ \\
\multirow{2}{*}{tolokers} & GCN & $3.72 \pm 0.76$ & $3.77 \pm 0.81$ & $3.75 \pm 0.81$ & $3.80 \pm 0.77$ & $3.40 \pm 1.00$ & $\boldsymbol{3.24 \pm 0.69}$ & $\underline{3.33 \pm 1.10}$ & $\textcolor{red}{2.29 \pm 0.68}$ & $4.13 \pm 0.75$ & $3.52 \pm 0.81$ \\
 & GAT & $4.88 \pm 1.06$ & $4.67 \pm 0.60$ & $3.57 \pm 0.79$ & $4.65 \pm 0.61$ & $\underline{2.45 \pm 0.58}$ & $\boldsymbol{2.09 \pm 0.42}$ & $4.18 \pm 0.42$ & $\textcolor{red}{1.79 \pm 0.45}$ & $4.38 \pm 0.50$ & $2.98 \pm 0.65$ \\
\midrule
\multicolumn{12}{l}{\textit{Large}} \\
\multirow{2}{*}{ogbn-arxiv} & GCN & $3.12 \pm 0.24$ & $3.14 \pm 0.22$ & $3.03 \pm 0.13$ & $3.08 \pm 0.23$ & $\boldsymbol{1.19 \pm 0.13}$ & $\underline{2.02 \pm 0.37}$ & $3.37 \pm 0.17$ & $5.09 \pm 0.56$ & $2.53 \pm 0.64$ & $\textcolor{red}{1.15 \pm 0.11}$ \\
 & GAT & $7.01 \pm 0.52$ & $2.16 \pm 0.13$ & $2.41 \pm 0.13$ & $2.06 \pm 0.17$ & $\textcolor{red}{1.40 \pm 0.15}$ & $\boldsymbol{1.62 \pm 0.20}$ & $1.89 \pm 0.25$ & $3.92 \pm 0.65$ & $\underline{1.85 \pm 0.17}$ & $1.91 \pm 0.17$ \\
\multirow{2}{*}{Reddit} & GCN & $6.93 \pm 0.14$ & $2.35 \pm 0.10$ & $2.39 \pm 0.08$ & $2.22 \pm 0.17$ & $\textcolor{red}{1.19 \pm 0.06}$ & $\underline{1.41 \pm 0.10}$ & -- & $3.22 \pm 0.18$ & $1.69 \pm 0.18$ & $\boldsymbol{1.35 \pm 0.13}$ \\
 & GAT & $3.77 \pm 0.20$ & $1.02 \pm 0.12$ & $1.28 \pm 0.20$ & $1.24 \pm 0.19$ & $1.09 \pm 0.16$ & $\textcolor{red}{0.48 \pm 0.10}$ & -- & $\boldsymbol{0.61 \pm 0.08}$ & $\underline{0.80 \pm 0.10}$ & $1.15 \pm 0.17$ \\
\bottomrule
\end{tabular}%
}
\end{table}

\clearpage
\begin{table}[t]
\centering
\setlength{\tabcolsep}{2pt}
\caption{Per-backbone degree-stratified ECE (\%, mean$\pm$std). Best per (dataset, backbone) row in \textcolor{red}{red}, second in \textbf{bold}, third \underline{underlined} (excluding Uncal).}
\label{tab:app_deg_ece}
\resizebox{\textwidth}{!}{%
\begin{tabular}{llcccccccccc}
\toprule
Dataset & Backbone & Uncal & TS & VS & ETS & HTS & CaGCN & GATS & GETS & WATS & HoTS \\
\midrule
\multicolumn{12}{l}{\textit{Homophilic}} \\
\multirow{2}{*}{Cora} & GCN & $21.47 \pm 1.15$ & $2.79 \pm 0.53$ & $\underline{2.72 \pm 0.82}$ & $2.85 \pm 0.58$ & $\boldsymbol{2.70 \pm 0.43}$ & $3.10 \pm 0.51$ & $2.89 \pm 0.57$ & $3.23 \pm 0.54$ & $2.83 \pm 0.54$ & $\textcolor{red}{2.59 \pm 0.41}$ \\
 & GAT & $17.79 \pm 2.37$ & $\boldsymbol{2.17 \pm 0.71}$ & $\textcolor{red}{2.09 \pm 0.73}$ & $2.35 \pm 0.75$ & $2.30 \pm 0.51$ & $2.72 \pm 0.73$ & $\boldsymbol{2.17 \pm 0.79}$ & $2.53 \pm 0.48$ & $2.25 \pm 0.57$ & $\boldsymbol{2.17 \pm 0.59}$ \\
\multirow{2}{*}{CiteSeer} & GCN & $21.86 \pm 3.45$ & $\textcolor{red}{2.71 \pm 0.43}$ & $\underline{2.76 \pm 0.59}$ & $\textcolor{red}{2.71 \pm 0.55}$ & $3.23 \pm 0.80$ & $3.32 \pm 0.54$ & $3.15 \pm 0.75$ & $4.04 \pm 0.91$ & $2.85 \pm 0.54$ & $2.91 \pm 0.93$ \\
 & GAT & $15.24 \pm 1.18$ & $\boldsymbol{2.58 \pm 0.49}$ & $2.83 \pm 0.52$ & $\textcolor{red}{2.46 \pm 0.39}$ & $3.36 \pm 0.76$ & $2.77 \pm 0.48$ & $2.94 \pm 0.87$ & $3.04 \pm 0.88$ & $\underline{2.67 \pm 0.39}$ & $3.13 \pm 0.45$ \\
\multirow{2}{*}{PubMed} & GCN & $13.94 \pm 1.52$ & $2.58 \pm 0.34$ & $2.92 \pm 0.34$ & $2.72 \pm 0.33$ & $\boldsymbol{1.83 \pm 0.36}$ & $\underline{1.84 \pm 1.47}$ & $2.19 \pm 0.39$ & $2.16 \pm 0.57$ & $\textcolor{red}{0.92 \pm 0.18}$ & $1.98 \pm 0.25$ \\
 & GAT & $10.40 \pm 0.83$ & $0.82 \pm 0.26$ & $\textcolor{red}{0.75 \pm 0.13}$ & $\boldsymbol{0.79 \pm 0.23}$ & $\underline{0.80 \pm 0.19}$ & $0.96 \pm 0.49$ & $\underline{0.80 \pm 0.29}$ & $1.02 \pm 0.38$ & $0.87 \pm 0.32$ & $0.84 \pm 0.29$ \\
\multirow{2}{*}{CoraFull} & GCN & $28.50 \pm 0.92$ & $4.83 \pm 0.46$ & $3.24 \pm 0.60$ & $4.89 \pm 0.46$ & $\underline{2.64 \pm 0.64}$ & $3.75 \pm 0.44$ & $4.89 \pm 0.46$ & $3.48 \pm 0.58$ & $\boldsymbol{2.33 \pm 0.55}$ & $\textcolor{red}{1.56 \pm 0.47}$ \\
 & GAT & $37.35 \pm 0.88$ & $\boldsymbol{1.81 \pm 0.57}$ & $2.36 \pm 0.24$ & $\textcolor{red}{1.42 \pm 0.36}$ & $3.06 \pm 0.97$ & $2.75 \pm 0.89$ & $\underline{1.94 \pm 0.64}$ & $2.30 \pm 1.30$ & $5.23 \pm 1.03$ & $2.15 \pm 0.96$ \\
\multirow{2}{*}{Computers} & GCN & $6.33 \pm 0.45$ & $4.02 \pm 0.52$ & $3.11 \pm 0.39$ & $4.06 \pm 0.52$ & $\underline{2.19 \pm 0.44}$ & $\boldsymbol{1.82 \pm 0.37}$ & $3.23 \pm 0.95$ & $3.36 \pm 0.64$ & $\textcolor{red}{1.40 \pm 0.32}$ & $2.76 \pm 0.48$ \\
 & GAT & $6.59 \pm 1.52$ & $2.19 \pm 0.27$ & $2.62 \pm 0.49$ & $2.19 \pm 0.25$ & $\underline{2.14 \pm 0.28}$ & $2.27 \pm 0.31$ & $\boldsymbol{2.12 \pm 0.29}$ & $2.83 \pm 0.94$ & $2.34 \pm 0.28$ & $\textcolor{red}{2.09 \pm 0.27}$ \\
\multirow{2}{*}{Photo} & GCN & $4.23 \pm 0.60$ & $1.97 \pm 0.28$ & $1.89 \pm 0.35$ & $1.99 \pm 0.31$ & $\boldsymbol{1.46 \pm 0.44}$ & $1.71 \pm 0.49$ & $1.94 \pm 0.25$ & $2.22 \pm 0.37$ & $\textcolor{red}{1.37 \pm 0.26}$ & $\underline{1.47 \pm 0.38}$ \\
 & GAT & $4.88 \pm 0.72$ & $1.67 \pm 0.41$ & $1.74 \pm 0.28$ & $1.86 \pm 0.60$ & $1.82 \pm 0.47$ & $\boldsymbol{1.49 \pm 0.33}$ & $1.72 \pm 0.46$ & $\textcolor{red}{1.35 \pm 0.22}$ & $\underline{1.55 \pm 0.34}$ & $1.72 \pm 0.40$ \\
\multirow{2}{*}{CS} & GCN & $1.72 \pm 0.21$ & $1.67 \pm 0.20$ & $1.66 \pm 0.17$ & $\underline{1.64 \pm 0.19}$ & $\boldsymbol{1.40 \pm 0.31}$ & $2.52 \pm 0.78$ & $1.72 \pm 0.19$ & $1.66 \pm 0.16$ & $3.45 \pm 0.49$ & $\textcolor{red}{1.15 \pm 0.20}$ \\
 & GAT & $3.58 \pm 0.73$ & $2.45 \pm 0.70$ & $2.25 \pm 0.52$ & $2.51 \pm 0.81$ & $\boldsymbol{2.15 \pm 0.37}$ & $\textcolor{red}{1.19 \pm 0.35}$ & $\boldsymbol{2.15 \pm 0.54}$ & $3.19 \pm 3.17$ & $2.56 \pm 0.78$ & $2.52 \pm 0.59$ \\
\multirow{2}{*}{Physics} & GCN & $0.98 \pm 0.10$ & $0.98 \pm 0.12$ & $0.97 \pm 0.11$ & $\underline{0.93 \pm 0.15}$ & $\boldsymbol{0.80 \pm 0.09}$ & $\underline{0.93 \pm 0.40}$ & $0.99 \pm 0.11$ & $1.00 \pm 0.14$ & $1.77 \pm 0.31$ & $\textcolor{red}{0.79 \pm 0.09}$ \\
 & GAT & $1.72 \pm 0.41$ & $0.80 \pm 0.26$ & $0.95 \pm 0.24$ & $0.90 \pm 0.27$ & $0.72 \pm 0.12$ & $\textcolor{red}{0.45 \pm 0.19}$ & $0.74 \pm 0.27$ & $0.76 \pm 0.57$ & $\boldsymbol{0.67 \pm 0.27}$ & $\underline{0.68 \pm 0.20}$ \\
\midrule
\multicolumn{12}{l}{\textit{Heterophilic}} \\
\multirow{2}{*}{Texas} & GCN & $27.24 \pm 7.36$ & $27.23 \pm 7.36$ & $25.64 \pm 5.45$ & $26.84 \pm 7.54$ & $\boldsymbol{24.57 \pm 7.08}$ & $25.94 \pm 7.40$ & $25.66 \pm 7.36$ & $\underline{25.22 \pm 3.60}$ & $28.26 \pm 7.01$ & $\textcolor{red}{23.07 \pm 6.59}$ \\
 & GAT & $22.51 \pm 6.55$ & $22.41 \pm 6.43$ & $22.33 \pm 5.74$ & $\underline{20.57 \pm 5.13}$ & $21.64 \pm 5.84$ & $21.53 \pm 7.23$ & $\boldsymbol{20.29 \pm 5.81}$ & $\textcolor{red}{19.99 \pm 2.65}$ & $23.93 \pm 5.91$ & $21.86 \pm 5.44$ \\
\multirow{2}{*}{Cornell} & GCN & $22.47 \pm 8.53$ & $21.50 \pm 7.66$ & $20.96 \pm 7.38$ & $19.38 \pm 7.61$ & $20.76 \pm 7.61$ & $\boldsymbol{18.06 \pm 7.45}$ & $19.84 \pm 6.95$ & $\underline{18.34 \pm 7.05}$ & $23.67 \pm 6.87$ & $\textcolor{red}{17.56 \pm 6.45}$ \\
 & GAT & $19.97 \pm 6.63$ & $20.04 \pm 6.68$ & $19.84 \pm 6.57$ & $19.40 \pm 6.37$ & $18.57 \pm 7.09$ & $\textcolor{red}{16.89 \pm 6.27}$ & $\underline{18.44 \pm 6.95}$ & $18.49 \pm 4.90$ & $24.89 \pm 5.57$ & $\boldsymbol{17.98 \pm 6.62}$ \\
\multirow{2}{*}{Wisconsin} & GCN & $19.96 \pm 7.01$ & $20.10 \pm 6.68$ & $21.57 \pm 4.96$ & $\boldsymbol{17.37 \pm 5.35}$ & $\underline{18.31 \pm 5.21}$ & $19.37 \pm 5.98$ & $18.37 \pm 5.07$ & $18.95 \pm 5.23$ & $24.29 \pm 8.28$ & $\textcolor{red}{16.52 \pm 4.10}$ \\
 & GAT & $18.88 \pm 6.28$ & $18.81 \pm 6.27$ & $19.18 \pm 6.14$ & $17.93 \pm 6.29$ & $17.28 \pm 5.32$ & $\boldsymbol{16.99 \pm 6.08}$ & $\underline{17.22 \pm 5.15}$ & $18.17 \pm 7.40$ & $23.67 \pm 5.96$ & $\textcolor{red}{15.96 \pm 3.34}$ \\
\multirow{2}{*}{Chameleon} & GCN & $12.07 \pm 2.00$ & $12.07 \pm 2.00$ & $11.98 \pm 2.08$ & $12.09 \pm 1.93$ & $11.88 \pm 1.59$ & $\textcolor{red}{10.60 \pm 1.47}$ & $\underline{11.79 \pm 1.84}$ & $\boldsymbol{11.75 \pm 2.21}$ & $13.62 \pm 2.16$ & $12.22 \pm 1.62$ \\
 & GAT & $12.56 \pm 2.07$ & $12.52 \pm 2.09$ & $12.33 \pm 2.01$ & $12.56 \pm 2.61$ & $12.29 \pm 2.15$ & $\boldsymbol{10.91 \pm 2.10}$ & $\underline{12.21 \pm 2.38}$ & $\textcolor{red}{10.88 \pm 2.01}$ & $15.28 \pm 2.80$ & $12.74 \pm 2.14$ \\
\multirow{2}{*}{Squirrel} & GCN & $6.27 \pm 1.20$ & $6.33 \pm 1.20$ & $6.42 \pm 1.00$ & $7.65 \pm 0.99$ & $\boldsymbol{5.45 \pm 1.09}$ & $\underline{5.51 \pm 1.08}$ & $7.09 \pm 1.09$ & $\textcolor{red}{5.36 \pm 1.03}$ & $6.96 \pm 1.97$ & $5.78 \pm 1.12$ \\
 & GAT & $6.48 \pm 1.97$ & $6.50 \pm 1.97$ & $6.78 \pm 1.63$ & $6.70 \pm 2.05$ & $6.35 \pm 2.10$ & $\boldsymbol{6.32 \pm 2.10}$ & $6.41 \pm 2.07$ & $\textcolor{red}{5.78 \pm 1.42}$ & $7.06 \pm 2.21$ & $\underline{6.34 \pm 2.15}$ \\
\multirow{2}{*}{Actor} & GCN & $2.36 \pm 0.69$ & $2.44 \pm 0.65$ & $\textcolor{red}{2.21 \pm 0.52}$ & $2.37 \pm 0.68$ & $2.49 \pm 0.74$ & $\boldsymbol{2.25 \pm 0.73}$ & $2.41 \pm 0.68$ & $2.96 \pm 1.38$ & $5.41 \pm 2.05$ & $\underline{2.27 \pm 0.72}$ \\
 & GAT & $1.78 \pm 0.35$ & $\underline{1.87 \pm 0.33}$ & $\textcolor{red}{1.62 \pm 0.48}$ & $1.88 \pm 0.27$ & $1.93 \pm 0.31$ & $1.91 \pm 0.42$ & $1.88 \pm 0.34$ & $4.03 \pm 1.85$ & $3.64 \pm 1.07$ & $\boldsymbol{1.79 \pm 0.40}$ \\
\multirow{2}{*}{Roman-Empire} & GCN & $12.57 \pm 0.67$ & $\boldsymbol{3.12 \pm 0.44}$ & $4.48 \pm 0.51$ & $\underline{3.15 \pm 0.51}$ & $3.40 \pm 0.49$ & $\textcolor{red}{3.01 \pm 0.47}$ & $3.32 \pm 0.53$ & $4.15 \pm 0.56$ & $3.65 \pm 0.85$ & $3.74 \pm 0.63$ \\
 & GAT & $12.08 \pm 1.28$ & $6.96 \pm 0.35$ & $\underline{6.12 \pm 0.52}$ & $7.04 \pm 0.34$ & $6.88 \pm 0.34$ & $\boldsymbol{5.42 \pm 0.28}$ & $6.68 \pm 0.31$ & $\textcolor{red}{2.45 \pm 0.67}$ & $6.70 \pm 0.28$ & $7.08 \pm 0.37$ \\
\multirow{2}{*}{tolokers} & GCN & $2.26 \pm 0.23$ & $\underline{2.31 \pm 0.36}$ & $2.34 \pm 0.27$ & $\boldsymbol{2.30 \pm 0.35}$ & $2.33 \pm 0.28$ & $3.19 \pm 0.67$ & $2.51 \pm 0.50$ & $2.75 \pm 0.57$ & $\textcolor{red}{2.24 \pm 0.37}$ & $2.47 \pm 0.21$ \\
 & GAT & $4.22 \pm 0.82$ & $3.62 \pm 0.29$ & $4.00 \pm 0.42$ & $3.62 \pm 0.29$ & $3.65 \pm 0.30$ & $\textcolor{red}{2.17 \pm 0.22}$ & $3.68 \pm 0.30$ & $\boldsymbol{2.56 \pm 0.33}$ & $\underline{2.83 \pm 0.35}$ & $3.60 \pm 0.43$ \\
\midrule
\multicolumn{12}{l}{\textit{Large}} \\
\multirow{2}{*}{ogbn-arxiv} & GCN & $1.67 \pm 0.20$ & $1.71 \pm 0.18$ & $\underline{1.62 \pm 0.12}$ & $1.71 \pm 0.18$ & $\boldsymbol{1.44 \pm 0.08}$ & $\textcolor{red}{1.37 \pm 0.28}$ & $2.50 \pm 0.26$ & $4.59 \pm 0.64$ & $1.93 \pm 0.57$ & $1.74 \pm 0.12$ \\
 & GAT & $6.95 \pm 0.46$ & $3.31 \pm 0.17$ & $3.86 \pm 0.12$ & $\underline{3.29 \pm 0.16}$ & $4.11 \pm 0.15$ & $\textcolor{red}{2.02 \pm 0.15}$ & $3.30 \pm 0.17$ & $3.38 \pm 0.75$ & $\boldsymbol{2.35 \pm 0.24}$ & $4.28 \pm 0.15$ \\
\multirow{2}{*}{Reddit} & GCN & $6.93 \pm 0.14$ & $1.86 \pm 0.10$ & $2.10 \pm 0.12$ & $1.91 \pm 0.10$ & $\textcolor{red}{1.21 \pm 0.06}$ & $1.52 \pm 0.11$ & -- & $2.31 \pm 0.26$ & $\underline{1.41 \pm 0.14}$ & $\boldsymbol{1.33 \pm 0.12}$ \\
 & GAT & $4.94 \pm 0.18$ & $3.24 \pm 0.10$ & $3.38 \pm 0.10$ & $3.44 \pm 0.10$ & $3.46 \pm 0.11$ & $\textcolor{red}{0.67 \pm 0.09}$ & -- & $\underline{1.16 \pm 0.15}$ & $\boldsymbol{0.83 \pm 0.16}$ & $3.53 \pm 0.12$ \\
\bottomrule
\end{tabular}%
}
\end{table}

\begin{table}[t]
\centering
\setlength{\tabcolsep}{2pt}
\caption{Per-backbone NLL (mean$\pm$std). Best per (dataset, backbone) row in \textcolor{red}{red}, second in \textbf{bold}, third \underline{underlined} (excluding Uncal).}
\label{tab:app_nll}
\resizebox{\textwidth}{!}{%
\begin{tabular}{llcccccccccc}
\toprule
Dataset & Backbone & Uncal & TS & VS & ETS & HTS & CaGCN & GATS & GETS & WATS & HoTS \\
\midrule
\multicolumn{12}{l}{\textit{Homophilic}} \\
\multirow{2}{*}{Cora} & GCN & $0.681 \pm 0.025$ & $0.518 \pm 0.031$ & $\textcolor{red}{0.481 \pm 0.026}$ & $0.512 \pm 0.030$ & $\underline{0.503 \pm 0.028}$ & $0.572 \pm 0.053$ & $0.518 \pm 0.030$ & $0.528 \pm 0.032$ & $0.516 \pm 0.030$ & $\boldsymbol{0.497 \pm 0.028}$ \\
 & GAT & $0.636 \pm 0.036$ & $0.516 \pm 0.027$ & $\boldsymbol{0.506 \pm 0.022}$ & $0.515 \pm 0.026$ & $\underline{0.513 \pm 0.024}$ & $0.553 \pm 0.026$ & $0.520 \pm 0.026$ & $0.523 \pm 0.021$ & $0.516 \pm 0.027$ & $\textcolor{red}{0.504 \pm 0.023}$ \\
\multirow{2}{*}{CiteSeer} & GCN & $0.977 \pm 0.066$ & $\underline{0.849 \pm 0.046}$ & $\textcolor{red}{0.809 \pm 0.023}$ & $\boldsymbol{0.844 \pm 0.042}$ & $0.853 \pm 0.052$ & $0.878 \pm 0.042$ & $0.853 \pm 0.046$ & $0.884 \pm 0.039$ & $\underline{0.849 \pm 0.048}$ & $0.859 \pm 0.050$ \\
 & GAT & $0.878 \pm 0.009$ & $\underline{0.801 \pm 0.012}$ & $\textcolor{red}{0.790 \pm 0.011}$ & $\boldsymbol{0.800 \pm 0.012}$ & $0.803 \pm 0.013$ & $0.832 \pm 0.017$ & $0.807 \pm 0.011$ & $0.832 \pm 0.016$ & $\underline{0.801 \pm 0.012}$ & $0.805 \pm 0.013$ \\
\multirow{2}{*}{PubMed} & GCN & $0.459 \pm 0.017$ & $0.380 \pm 0.005$ & $0.377 \pm 0.005$ & $0.377 \pm 0.005$ & $0.376 \pm 0.006$ & $\boldsymbol{0.368 \pm 0.017}$ & $0.375 \pm 0.006$ & $\underline{0.371 \pm 0.005}$ & $\textcolor{red}{0.365 \pm 0.005}$ & $0.376 \pm 0.005$ \\
 & GAT & $0.430 \pm 0.005$ & $0.378 \pm 0.006$ & $\boldsymbol{0.375 \pm 0.005}$ & $0.378 \pm 0.006$ & $0.379 \pm 0.006$ & $\boldsymbol{0.375 \pm 0.005}$ & $0.379 \pm 0.006$ & $\textcolor{red}{0.373 \pm 0.005}$ & $0.379 \pm 0.006$ & $0.378 \pm 0.005$ \\
\multirow{2}{*}{CoraFull} & GCN & $1.947 \pm 0.016$ & $1.590 \pm 0.030$ & $\textcolor{red}{1.324 \pm 0.017}$ & $1.582 \pm 0.029$ & $1.572 \pm 0.029$ & $1.563 \pm 0.030$ & $1.590 \pm 0.030$ & $\boldsymbol{1.395 \pm 0.034}$ & $\underline{1.562 \pm 0.030}$ & $1.569 \pm 0.028$ \\
 & GAT & $2.256 \pm 0.022$ & $1.667 \pm 0.026$ & $\textcolor{red}{1.258 \pm 0.023}$ & $1.662 \pm 0.026$ & $1.684 \pm 0.028$ & $\underline{1.648 \pm 0.035}$ & $1.668 \pm 0.026$ & $\boldsymbol{1.439 \pm 0.039}$ & $1.689 \pm 0.029$ & $1.673 \pm 0.028$ \\
\multirow{2}{*}{Computers} & GCN & $0.400 \pm 0.012$ & $0.389 \pm 0.015$ & $\underline{0.367 \pm 0.011}$ & $0.381 \pm 0.015$ & $0.373 \pm 0.014$ & $\boldsymbol{0.363 \pm 0.013}$ & $0.381 \pm 0.020$ & $0.375 \pm 0.007$ & $\textcolor{red}{0.360 \pm 0.013}$ & $0.371 \pm 0.017$ \\
 & GAT & $0.480 \pm 0.014$ & $0.455 \pm 0.020$ & $\textcolor{red}{0.400 \pm 0.007}$ & $0.453 \pm 0.021$ & $0.452 \pm 0.019$ & $\underline{0.442 \pm 0.017}$ & $0.457 \pm 0.021$ & $\boldsymbol{0.406 \pm 0.016}$ & $0.453 \pm 0.020$ & $0.450 \pm 0.020$ \\
\multirow{2}{*}{Photo} & GCN & $0.251 \pm 0.010$ & $0.242 \pm 0.013$ & $\textcolor{red}{0.235 \pm 0.010}$ & $\underline{0.239 \pm 0.013}$ & $0.240 \pm 0.012$ & $0.246 \pm 0.015$ & $0.242 \pm 0.013$ & $0.244 \pm 0.011$ & $\underline{0.239 \pm 0.018}$ & $\boldsymbol{0.238 \pm 0.009}$ \\
 & GAT & $0.308 \pm 0.015$ & $0.297 \pm 0.018$ & $\boldsymbol{0.278 \pm 0.014}$ & $\underline{0.285 \pm 0.015}$ & $0.297 \pm 0.018$ & $0.291 \pm 0.020$ & $0.302 \pm 0.019$ & $\textcolor{red}{0.270 \pm 0.013}$ & $0.294 \pm 0.026$ & $0.297 \pm 0.018$ \\
\multirow{2}{*}{CS} & GCN & $0.275 \pm 0.008$ & $0.276 \pm 0.008$ & $\underline{0.275 \pm 0.008}$ & $\underline{0.275 \pm 0.008}$ & $\textcolor{red}{0.268 \pm 0.008}$ & $0.340 \pm 0.048$ & $0.276 \pm 0.009$ & $0.278 \pm 0.008$ & $0.308 \pm 0.017$ & $\textcolor{red}{0.268 \pm 0.008}$ \\
 & GAT & $0.385 \pm 0.011$ & $0.384 \pm 0.014$ & $0.375 \pm 0.011$ & $0.382 \pm 0.015$ & $\textcolor{red}{0.368 \pm 0.013}$ & $\textcolor{red}{0.368 \pm 0.027}$ & $0.386 \pm 0.015$ & $0.421 \pm 0.145$ & $0.414 \pm 0.026$ & $\textcolor{red}{0.368 \pm 0.013}$ \\
\multirow{2}{*}{Physics} & GCN & $0.140 \pm 0.004$ & $0.140 \pm 0.004$ & $\underline{0.139 \pm 0.004}$ & $\underline{0.139 \pm 0.004}$ & $\textcolor{red}{0.138 \pm 0.004}$ & $0.149 \pm 0.015$ & $\underline{0.139 \pm 0.004}$ & $0.141 \pm 0.007$ & $0.153 \pm 0.007$ & $\textcolor{red}{0.138 \pm 0.004}$ \\
 & GAT & $0.184 \pm 0.003$ & $0.182 \pm 0.004$ & $\underline{0.177 \pm 0.003}$ & $0.178 \pm 0.005$ & $\boldsymbol{0.176 \pm 0.004}$ & $\textcolor{red}{0.167 \pm 0.003}$ & $0.181 \pm 0.005$ & $0.192 \pm 0.051$ & $0.184 \pm 0.008$ & $\underline{0.177 \pm 0.004}$ \\
\midrule
\multicolumn{12}{l}{\textit{Heterophilic}} \\
\multirow{2}{*}{Texas} & GCN & $1.662 \pm 0.354$ & $1.658 \pm 0.349$ & $1.566 \pm 0.396$ & $\underline{1.517 \pm 0.145}$ & $\underline{1.517 \pm 0.231}$ & $1.587 \pm 0.179$ & $\boldsymbol{1.514 \pm 0.166}$ & $1.701 \pm 0.279$ & $1.911 \pm 0.532$ & $\textcolor{red}{1.402 \pm 0.088}$ \\
 & GAT & $1.452 \pm 0.266$ & $1.443 \pm 0.261$ & $1.396 \pm 0.223$ & $\underline{1.286 \pm 0.174}$ & $1.301 \pm 0.138$ & $1.581 \pm 0.745$ & $\boldsymbol{1.272 \pm 0.099}$ & $1.368 \pm 0.234$ & $1.644 \pm 0.350$ & $\textcolor{red}{1.227 \pm 0.083}$ \\
\multirow{2}{*}{Cornell} & GCN & $1.707 \pm 0.169$ & $1.687 \pm 0.148$ & $1.633 \pm 0.175$ & $\boldsymbol{1.563 \pm 0.068}$ & $\underline{1.595 \pm 0.080}$ & $1.711 \pm 0.309$ & $1.653 \pm 0.218$ & $2.012 \pm 0.520$ & $1.864 \pm 0.275$ & $\textcolor{red}{1.514 \pm 0.089}$ \\
 & GAT & $1.595 \pm 0.401$ & $1.595 \pm 0.392$ & $1.580 \pm 0.400$ & $\boldsymbol{1.459 \pm 0.139}$ & $\underline{1.532 \pm 0.235}$ & $1.638 \pm 0.559$ & $\underline{1.532 \pm 0.182}$ & $1.624 \pm 0.273$ & $1.811 \pm 0.580$ & $\textcolor{red}{1.426 \pm 0.087}$ \\
\multirow{2}{*}{Wisconsin} & GCN & $1.674 \pm 0.317$ & $1.673 \pm 0.309$ & $1.617 \pm 0.328$ & $\boldsymbol{1.508 \pm 0.079}$ & $\underline{1.537 \pm 0.163}$ & $1.768 \pm 0.480$ & $1.580 \pm 0.170$ & $1.616 \pm 0.185$ & $2.026 \pm 0.668$ & $\textcolor{red}{1.454 \pm 0.082}$ \\
 & GAT & $1.656 \pm 0.597$ & $1.649 \pm 0.589$ & $1.644 \pm 0.590$ & $1.527 \pm 0.610$ & $\underline{1.480 \pm 0.341}$ & $1.606 \pm 0.478$ & $\boldsymbol{1.473 \pm 0.382}$ & $1.667 \pm 0.579$ & $2.114 \pm 1.071$ & $\textcolor{red}{1.321 \pm 0.041}$ \\
\multirow{2}{*}{Chameleon} & GCN & $1.369 \pm 0.023$ & $1.368 \pm 0.023$ & $1.368 \pm 0.020$ & $\underline{1.333 \pm 0.021}$ & $\boldsymbol{1.327 \pm 0.035}$ & $\underline{1.333 \pm 0.056}$ & $1.349 \pm 0.030$ & $1.457 \pm 0.164$ & $1.480 \pm 0.063$ & $\textcolor{red}{1.317 \pm 0.029}$ \\
 & GAT & $1.270 \pm 0.031$ & $1.268 \pm 0.031$ & $1.268 \pm 0.031$ & $\boldsymbol{1.206 \pm 0.063}$ & $1.221 \pm 0.044$ & $\textcolor{red}{1.198 \pm 0.045}$ & $1.221 \pm 0.057$ & $1.238 \pm 0.065$ & $1.421 \pm 0.087$ & $\underline{1.216 \pm 0.048}$ \\
\multirow{2}{*}{Squirrel} & GCN & $1.474 \pm 0.014$ & $1.474 \pm 0.014$ & $1.472 \pm 0.014$ & $1.456 \pm 0.020$ & $\boldsymbol{1.442 \pm 0.014}$ & $\textcolor{red}{1.432 \pm 0.013}$ & $1.460 \pm 0.014$ & $1.716 \pm 0.132$ & $1.526 \pm 0.035$ & $\underline{1.445 \pm 0.015}$ \\
 & GAT & $1.536 \pm 0.016$ & $1.536 \pm 0.016$ & $1.536 \pm 0.016$ & $\textcolor{red}{1.527 \pm 0.019}$ & $\underline{1.531 \pm 0.016}$ & $1.532 \pm 0.015$ & $\underline{1.531 \pm 0.017}$ & $1.565 \pm 0.030$ & $1.564 \pm 0.022$ & $\boldsymbol{1.529 \pm 0.019}$ \\
\multirow{2}{*}{Actor} & GCN & $1.548 \pm 0.006$ & $\underline{1.548 \pm 0.006}$ & $\boldsymbol{1.547 \pm 0.005}$ & $\underline{1.548 \pm 0.006}$ & $\underline{1.548 \pm 0.006}$ & $\underline{1.548 \pm 0.007}$ & $1.549 \pm 0.006$ & $\textcolor{red}{1.542 \pm 0.011}$ & $1.564 \pm 0.012$ & $1.551 \pm 0.006$ \\
 & GAT & $1.557 \pm 0.004$ & $1.557 \pm 0.004$ & $\boldsymbol{1.547 \pm 0.005}$ & $1.556 \pm 0.004$ & $1.556 \pm 0.004$ & $\underline{1.555 \pm 0.003}$ & $1.556 \pm 0.004$ & $\textcolor{red}{1.524 \pm 0.009}$ & $1.567 \pm 0.014$ & $1.557 \pm 0.004$ \\
\multirow{2}{*}{Roman-Empire} & GCN & $1.627 \pm 0.024$ & $1.559 \pm 0.020$ & $\boldsymbol{1.494 \pm 0.016}$ & $1.559 \pm 0.020$ & $1.567 \pm 0.021$ & $\underline{1.544 \pm 0.020}$ & $1.566 \pm 0.019$ & $\textcolor{red}{1.455 \pm 0.037}$ & $1.562 \pm 0.019$ & $1.565 \pm 0.020$ \\
 & GAT & $2.019 \pm 0.032$ & $1.969 \pm 0.035$ & $\boldsymbol{1.898 \pm 0.039}$ & $\underline{1.968 \pm 0.035}$ & $1.972 \pm 0.035$ & $1.981 \pm 0.033$ & $\underline{1.968 \pm 0.035}$ & $\textcolor{red}{1.463 \pm 0.039}$ & $1.974 \pm 0.037$ & $1.972 \pm 0.035$ \\
\multirow{2}{*}{tolokers} & GCN & $0.469 \pm 0.005$ & $0.468 \pm 0.005$ & $0.468 \pm 0.005$ & $0.468 \pm 0.005$ & $0.467 \pm 0.005$ & $\boldsymbol{0.460 \pm 0.005}$ & $\underline{0.464 \pm 0.004}$ & $\textcolor{red}{0.432 \pm 0.011}$ & $0.466 \pm 0.005$ & $0.469 \pm 0.003$ \\
 & GAT & $0.492 \pm 0.005$ & $0.490 \pm 0.004$ & $0.484 \pm 0.004$ & $0.490 \pm 0.004$ & $\underline{0.482 \pm 0.005}$ & $\boldsymbol{0.465 \pm 0.003}$ & $0.490 \pm 0.004$ & $\textcolor{red}{0.429 \pm 0.013}$ & $0.483 \pm 0.004$ & $\underline{0.482 \pm 0.003}$ \\
\midrule
\multicolumn{12}{l}{\textit{Large}} \\
\multirow{2}{*}{ogbn-arxiv} & GCN & $0.987 \pm 0.005$ & $0.987 \pm 0.005$ & $0.980 \pm 0.003$ & $0.984 \pm 0.005$ & $0.974 \pm 0.004$ & $\textcolor{red}{0.962 \pm 0.011}$ & $0.991 \pm 0.005$ & $0.999 \pm 0.008$ & $\boldsymbol{0.971 \pm 0.012}$ & $\underline{0.972 \pm 0.004}$ \\
 & GAT & $1.239 \pm 0.003$ & $1.217 \pm 0.006$ & $\textcolor{red}{1.191 \pm 0.003}$ & $1.216 \pm 0.006$ & $\underline{1.209 \pm 0.007}$ & $\boldsymbol{1.202 \pm 0.006}$ & $1.213 \pm 0.006$ & $1.210 \pm 0.005$ & $1.210 \pm 0.006$ & $1.210 \pm 0.007$ \\
\multirow{2}{*}{Reddit} & GCN & $0.355 \pm 0.003$ & $0.321 \pm 0.004$ & $0.315 \pm 0.004$ & $0.314 \pm 0.007$ & $\boldsymbol{0.301 \pm 0.004}$ & $\textcolor{red}{0.299 \pm 0.004}$ & -- & $0.346 \pm 0.004$ & $0.306 \pm 0.008$ & $\boldsymbol{0.301 \pm 0.004}$ \\
 & GAT & $0.313 \pm 0.007$ & $0.306 \pm 0.007$ & $0.303 \pm 0.006$ & $0.296 \pm 0.006$ & $0.303 \pm 0.007$ & $\boldsymbol{0.270 \pm 0.006}$ & -- & $\textcolor{red}{0.256 \pm 0.004}$ & $\underline{0.273 \pm 0.006}$ & $0.304 \pm 0.008$ \\
\bottomrule
\end{tabular}%
}
\end{table}

\clearpage
\subsection{Calibration Across GNN Backbones}
\label{app:cross_backbone}

Table~\ref{tab:cross_backbone} extends the GCN and GAT evaluations in Appendix~\ref{app:per_backbone} to GCNII, GPR-GNN, GIN, and GraphSAGE. Each backbone is evaluated on the same 17 datasets and 10 seeds under the random train/validation/test split protocol, using the same calibration baselines as the main comparison. Within each dataset, backbone, and seed, all calibrators use the same frozen classifier logits. HoTS uses the observed-label-only predictor described in Appendix~\ref{app:hots_predictor}. Reddit is excluded because some of the compared pipelines exceed the available memory. The table reports the macro-average of the per-dataset mean ECE values.

\begin{table}[ht]
\centering
\small
\setlength{\tabcolsep}{3pt}
\caption{Calibration across additional GNN backbones. Mean ECE (\%) over 17 datasets, with 10 seeds per dataset. Reddit is excluded. Lower is better; best per row in \textbf{bold}.}
\label{tab:cross_backbone}
\resizebox{\textwidth}{!}{%
\begin{tabular}{lccccccccc}
\toprule
Backbone & TS & VS & ETS & HTS & CaGCN & GATS & GETS & WATS & HoTS \\
\midrule
GCNII & 6.82 & 6.08 & 6.29 & 6.23 & \textbf{5.79} & 6.35 & 6.79 & 7.74 & 5.84 \\
GPR-GNN & 6.85 & 5.95 & 6.47 & 5.73 & 6.05 & 5.87 & 7.11 & 7.78 & \textbf{5.49} \\
GIN & 6.32 & 5.67 & 5.37 & 5.23 & 6.37 & 5.31 & 6.33 & 8.98 & \textbf{4.97} \\
GraphSAGE & 6.60 & 6.10 & 5.53 & 5.75 & 6.63 & 5.79 & 6.51 & 9.20 & \textbf{5.11} \\
\bottomrule
\end{tabular}%
}
\end{table}

HoTS achieves the lowest mean ECE on GPR-GNN, GIN, and GraphSAGE, and ranks second on GCNII. These results support its applicability across different GNN architectures.

\clearpage
\subsection{Reliability Diagrams and Node-Level Diagnostics}
\label{app:reliability_diagnostics}

Figure~\ref{fig:reliability} shows a representative reliability diagram on Reddit (GCN backbone). HoTS reduces the visible calibration gap left after standard temperature scaling. Pooled ECE decreases from $6.93\%$ before calibration to $1.34\%$ after calibration, compared with $2.34\%$ for TS. Figure~\ref{fig:reliability_grid_pooled} extends this view across all benchmarks.

\begin{figure}[h]
\centering
\includegraphics[width=0.9\textwidth]{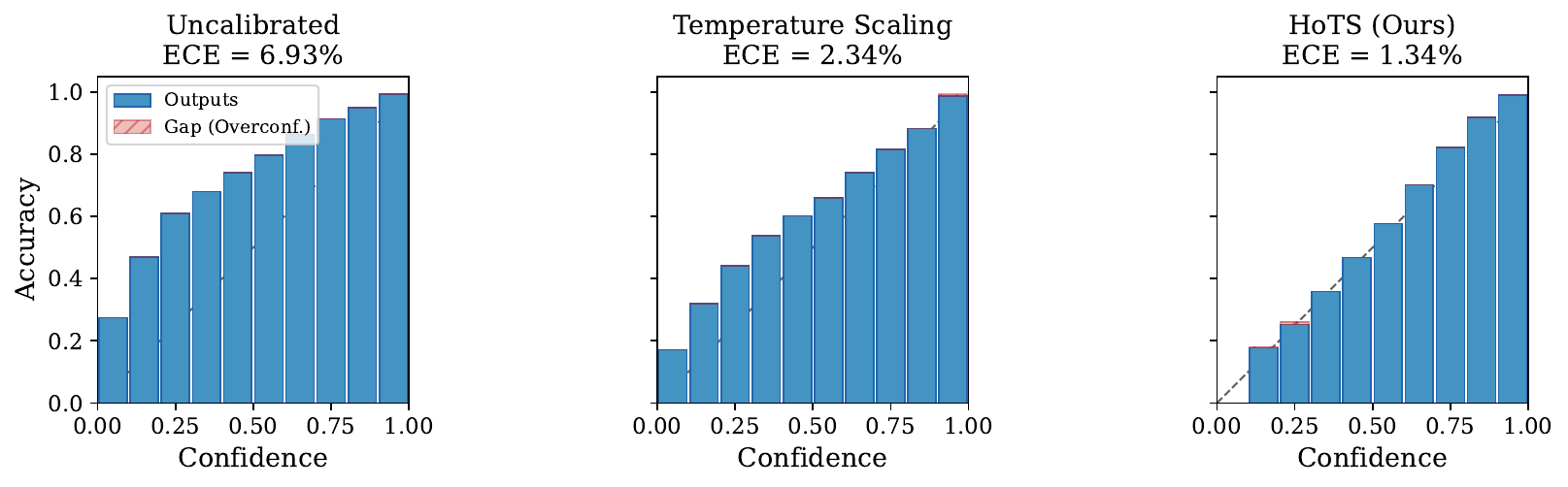}
\caption{Reliability diagrams on Reddit with a GCN backbone, pooling test nodes across 10 seeds into 10 equal-width confidence bins. TS reduces the calibration gap, and HoTS further improves calibration by allowing the temperature to vary across nodes.}
\label{fig:reliability}
\end{figure}

\begin{figure}[!h]
\centering
\includegraphics[width=0.85\textwidth]{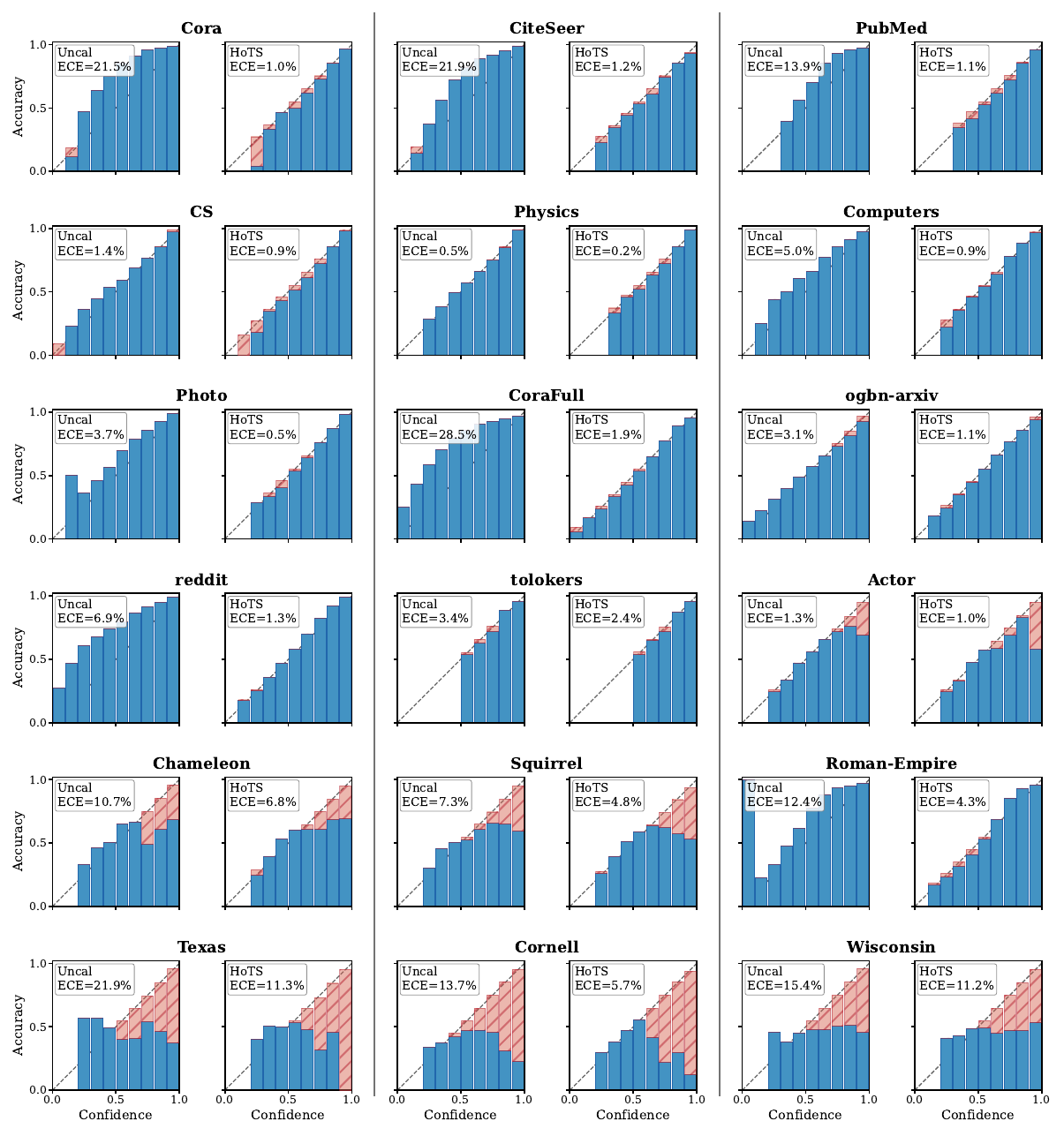}
\caption{Reliability diagrams across the 18 benchmark datasets. Each pair compares the uncalibrated GCN and HoTS, pooling test nodes across 10 seeds into 10 equal-width confidence bins. The figure is intended as a diagnostic complement to Table~\ref{tab:main_ece}; pooled ECE differs from the mean per-seed ECE reported in the tables.}
\label{fig:reliability_grid_pooled}
\end{figure}

\clearpage

Figure~\ref{fig:scatter} visualizes node-level confidence gaps on representative homophilic and heterophilic graphs. The horizontal axis evaluates the homophily-dependent temperature correction using oracle homophily, and the vertical axis is predicted confidence minus the node's correctness indicator. Coloring by oracle homophily and by entropy complement relates these gaps to neighborhood composition and raw logit concentration. Oracle homophily is computed from true labels for this retrospective diagnostic; the implemented HoTS temperature map uses estimated homophily.

\begin{figure}[h]
\centering
\includegraphics[width=\textwidth]{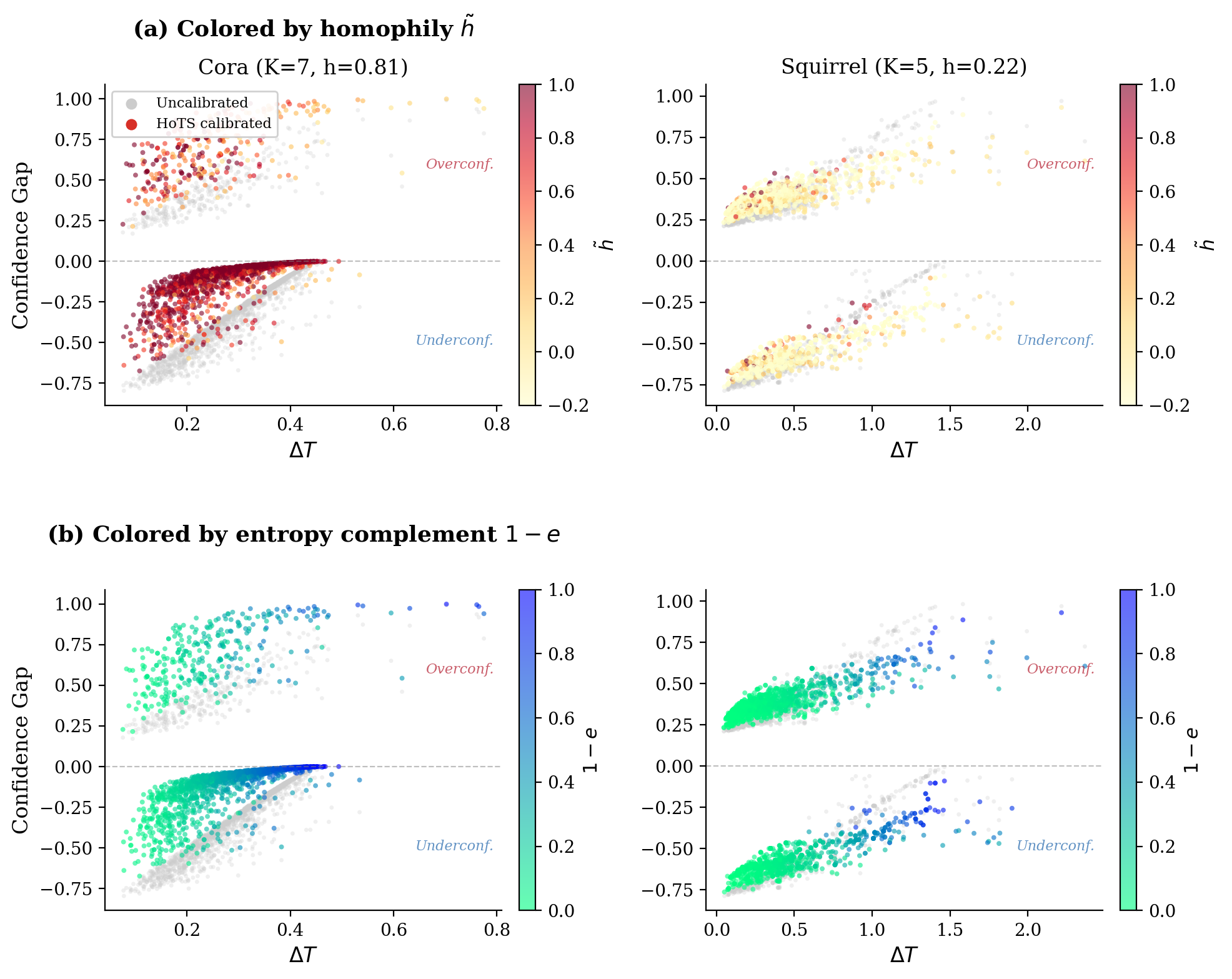}
\caption{Node-level diagnostic for HoTS on representative datasets. Points show confidence gap versus the homophily-dependent temperature correction evaluated with oracle homophily. The two rows color the same nodes by oracle normalized homophily $\tilde h$ and entropy complement $1-e$. Oracle homophily is computed from true labels solely for this diagnostic.}
\label{fig:scatter}
\end{figure}

Figures~\ref{fig:scatter_app_h} and \ref{fig:scatter_app_e} provide the same diagnostic across all datasets.

\begin{figure}[h]
\centering
\includegraphics[width=\textwidth]{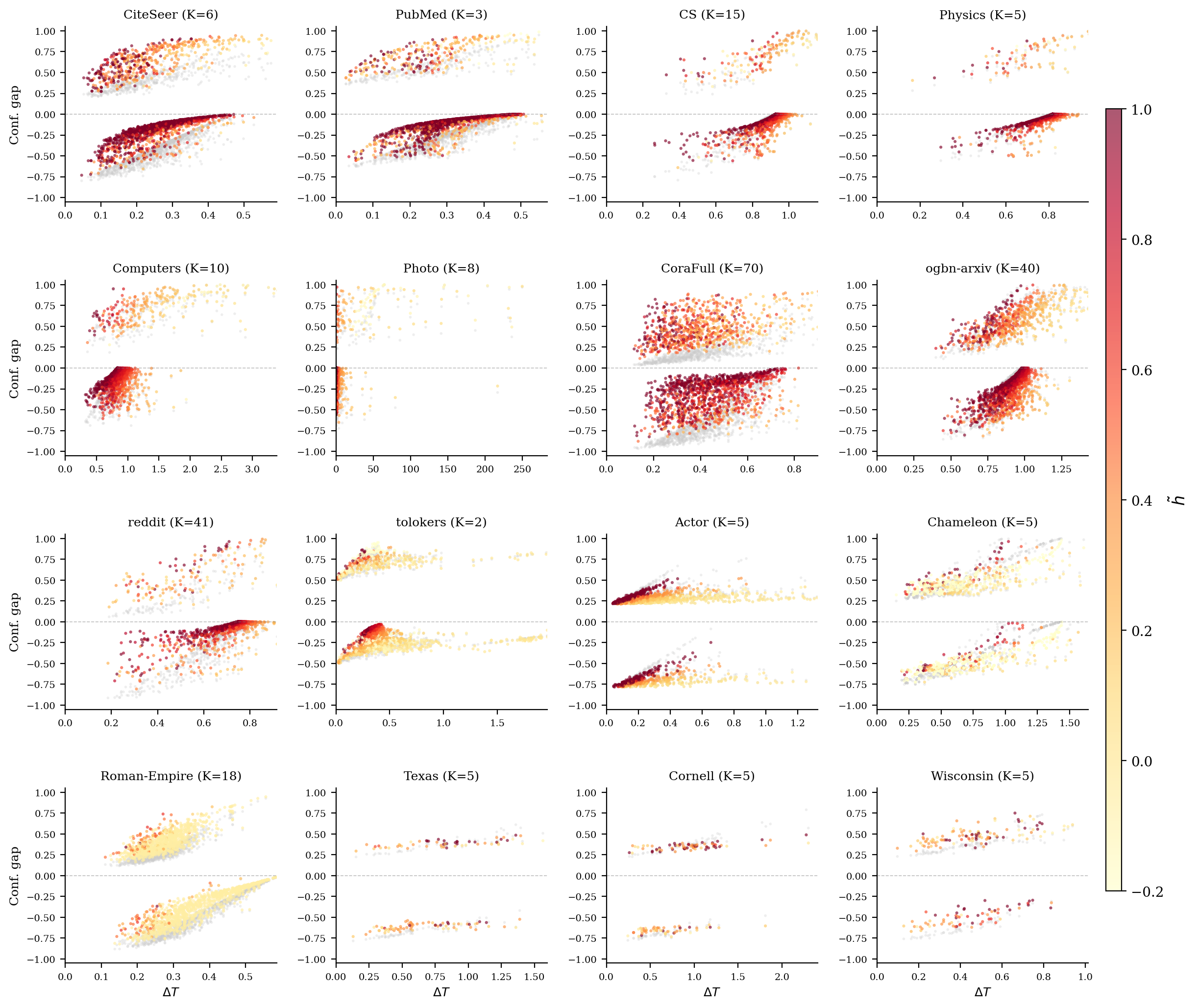}
\caption{Node-level diagnostic across datasets, colored by oracle normalized homophily $\tilde h$, computed from true labels for this diagnostic. The effect of HoTS is most structured on datasets where local homophily varies substantially across nodes.}
\label{fig:scatter_app_h}
\end{figure}

\begin{figure}[h]
\centering
\includegraphics[width=\textwidth]{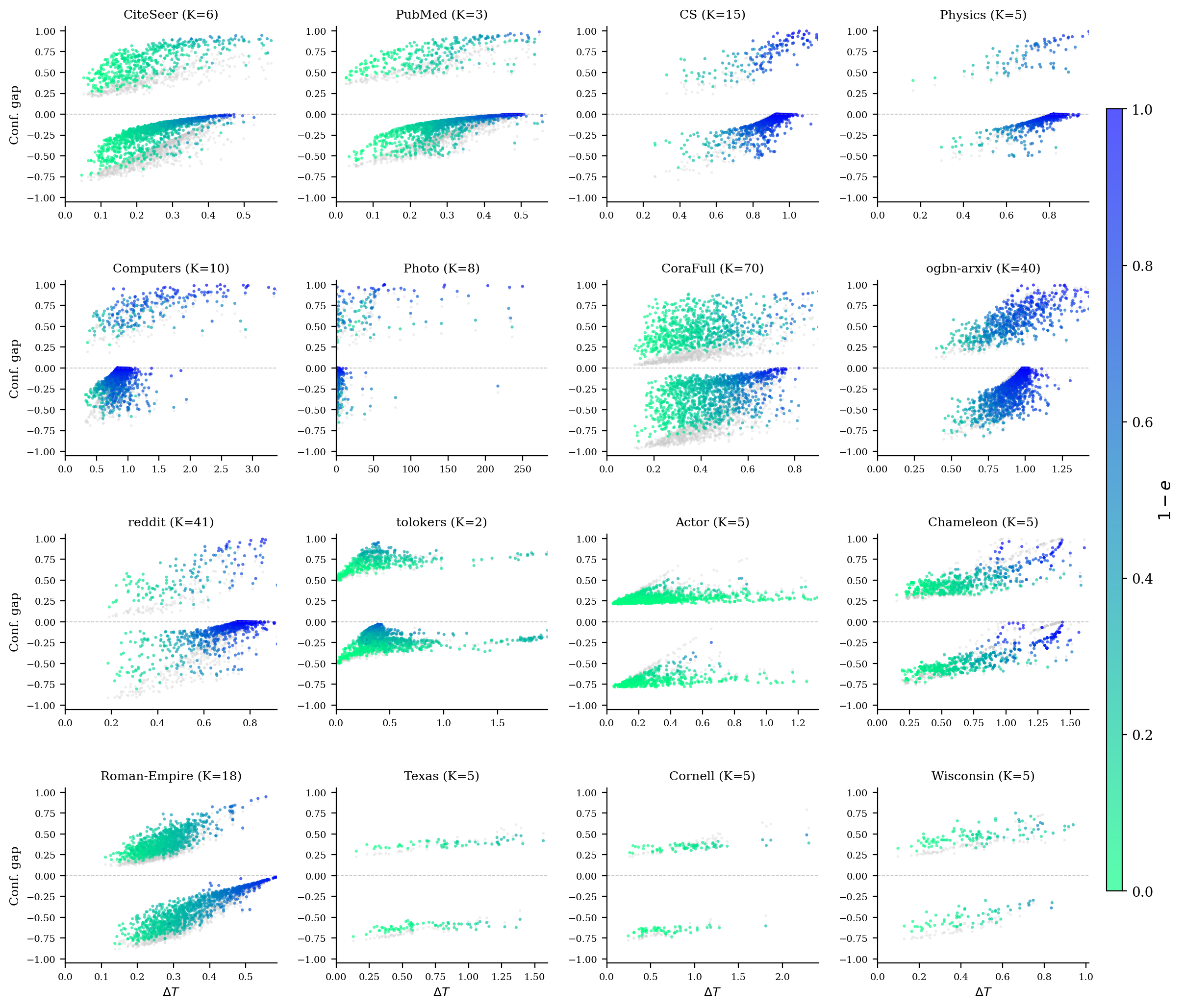}
\caption{Node-level diagnostic across datasets, colored by entropy complement $1-e$. Entropy complement captures logit concentration and is complementary to the homophily signal in Figure~\ref{fig:scatter_app_h}.}
\label{fig:scatter_app_e}
\end{figure}

\clearpage
\subsection{KL Bound Validation}
\label{app:kl_validation}

Appendix~\ref{app:sparse_overlap} derives that, in the sparse-overlap regime, the Gaussian residual KL from the joint neighbor-logit law to the product law scales as $O(K^2\bar d^2/n)$. Figure~\ref{fig:kl_bound} checks this scaling on CSBM graphs using three independent sweeps. The fitted slopes are close to the predicted exponents. The sweep over $K$ has slope $1.81$ overall and $1.97$ when restricted to $K\ge8$, the sweep over $\bar d$ has slope $2.08$, and the sweep over $n$ has slope $-1.01$. A discrimination test at fixed $n_{\mathrm{pc}}$ favors quadratic rather than linear scaling in $K$.

\begin{figure}[h]
\centering
\includegraphics[width=0.5\textwidth]{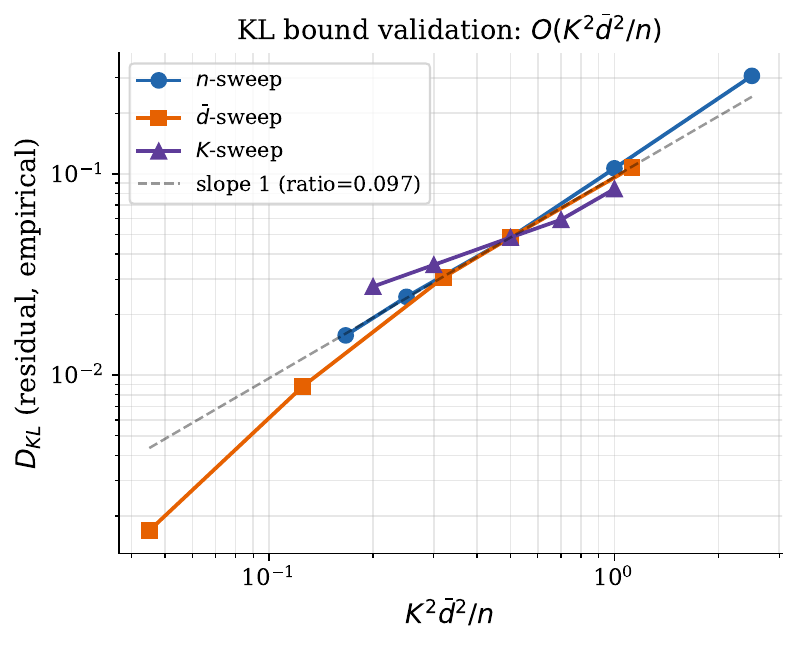}
\caption{KL bound validation on CSBM. Three independent sweeps over $n$, $\bar d$, and $K$ collapse when plotted against $K^2\bar d^2/n$, supporting the residual scaling derived in Appendix~\ref{app:sparse_overlap}.}
\label{fig:kl_bound}
\end{figure}

\subsection{Structural Perturbation Stress Test}
\label{app:shift}

We also evaluate a simple structural stress test. Edges are randomly added or removed at rates from 5\% to 50\%. This auxiliary experiment uses predicted-label neighbor agreement to estimate homophily, rather than the learned GCN estimator used in the main experiments. Logits and this estimate are recomputed on each perturbed graph, and TS and HoTS are refitted on its calibration split. This is not a full cross-domain distribution shift study, but it tests whether the structural feature used by HoTS becomes brittle under moderate graph perturbations. Table~\ref{tab:shift} reports ECE on representative homophilic and heterophilic benchmarks under increasing perturbation rates.

\begin{table}[h]                                                                                      
  \centering\small                                                                                      
  \caption{ECE (\%) under edge perturbation. HoTS is comparatively stable on the homophilic benchmarks
  shown here. On heterophilic graphs, random perturbation can improve calibration for both TS and HoTS  
  by weakening the original heterophilic structure.}                                                    
  \label{tab:shift}                                                                                     
  \begin{tabular}{llcccccc}                                                                             
  \toprule                                                                                              
  Dataset & Method & 0\% & 5\% & 10\% & 20\% & 30\% & 50\% \\                                           
  \midrule                                                                                            
  \multirow{2}{*}{Cora} & HoTS & 2.27 & 2.72 & 2.21 & 2.55 & 2.77 & 3.03 \\                           
   & TS & 3.39 & 3.24 & 2.81 & 3.73 & 3.69 & 4.16 \\                                                    
  \midrule                                    
  \multirow{2}{*}{CiteSeer} & HoTS & 2.83 & 3.09 & 3.21 & 3.10 & 2.72 & 3.92 \\                         
   & TS & 3.09 & 2.95 & 3.09 & 3.11 & 3.13 & 3.79 \\                                                    
  \midrule                                                                                              
  \multirow{2}{*}{CS} & HoTS & 0.99 & 0.95 & 0.95 & 1.26 & 1.19 & 1.15 \\                               
   & TS & 1.49 & 1.49 & 1.54 & 1.93 & 1.89 & 2.44 \\                                                    
  \midrule                                                                                              
  \multirow{2}{*}{Chameleon} & HoTS & 9.58 & 6.50 & 7.28 & 6.33 & 5.90 & 5.36 \\                        
   & TS & 12.61 & 10.00 & 10.22 & 7.45 & 5.90 & 3.29 \\                                                 
  \midrule                                                                                              
  \multirow{2}{*}{Squirrel} & HoTS & 5.09 & 3.27 & 2.46 & 1.68 & 1.25 & 0.66 \\                       
   & TS & 7.85 & 7.30 & 4.71 & 3.02 & 2.10 & 3.93 \\                                                    
  \bottomrule                                                                                           
  \end{tabular}                                                                                         
  \end{table}  

On Cora and CS, HoTS changes by less than one percentage point at 50\% perturbation, whereas TS degrades more substantially. On Chameleon and Squirrel, calibration improves for both methods as random perturbations partially wash out the original heterophilic structure. This mixed behavior is consistent with the role of $\alpha$. When estimated homophily is informative, HoTS can exploit it. When the structural signal is weakened or altered, the learned exponent and the base temperature reduce the reliance on that signal.

\subsection{Confidence-Conditioned Homophily Diagnostics}
\label{app:homophily_diagnostics}

We use oracle local homophily to examine whether neighborhood label composition is associated with prediction correctness within base-model confidence bins. This choice isolates the structural quantity motivating Proposition~\ref{prop:necessity} from errors introduced by the learned homophily estimator. Oracle homophily $h_i$ is the fraction of same-label neighbors on the backbone input graph, including self-loops. We pool GCN test predictions over ten seeds per dataset, stratify nodes into five confidence-quantile bins, and compare empirical accuracy across oracle-homophily groups within each bin. Figure~\ref{fig:p1_homophily_split} uses a 50/50 split, while Figure~\ref{fig:p1_homophily_split_3bin} uses three groups. True labels are used solely for this retrospective grouping; these oracle values are not used to fit the homophily estimator or the temperature map. The implemented HoTS uses the observed-label-only estimate $\hat h_i$.

\begin{figure}[h]
\centering
\includegraphics[width=\textwidth]{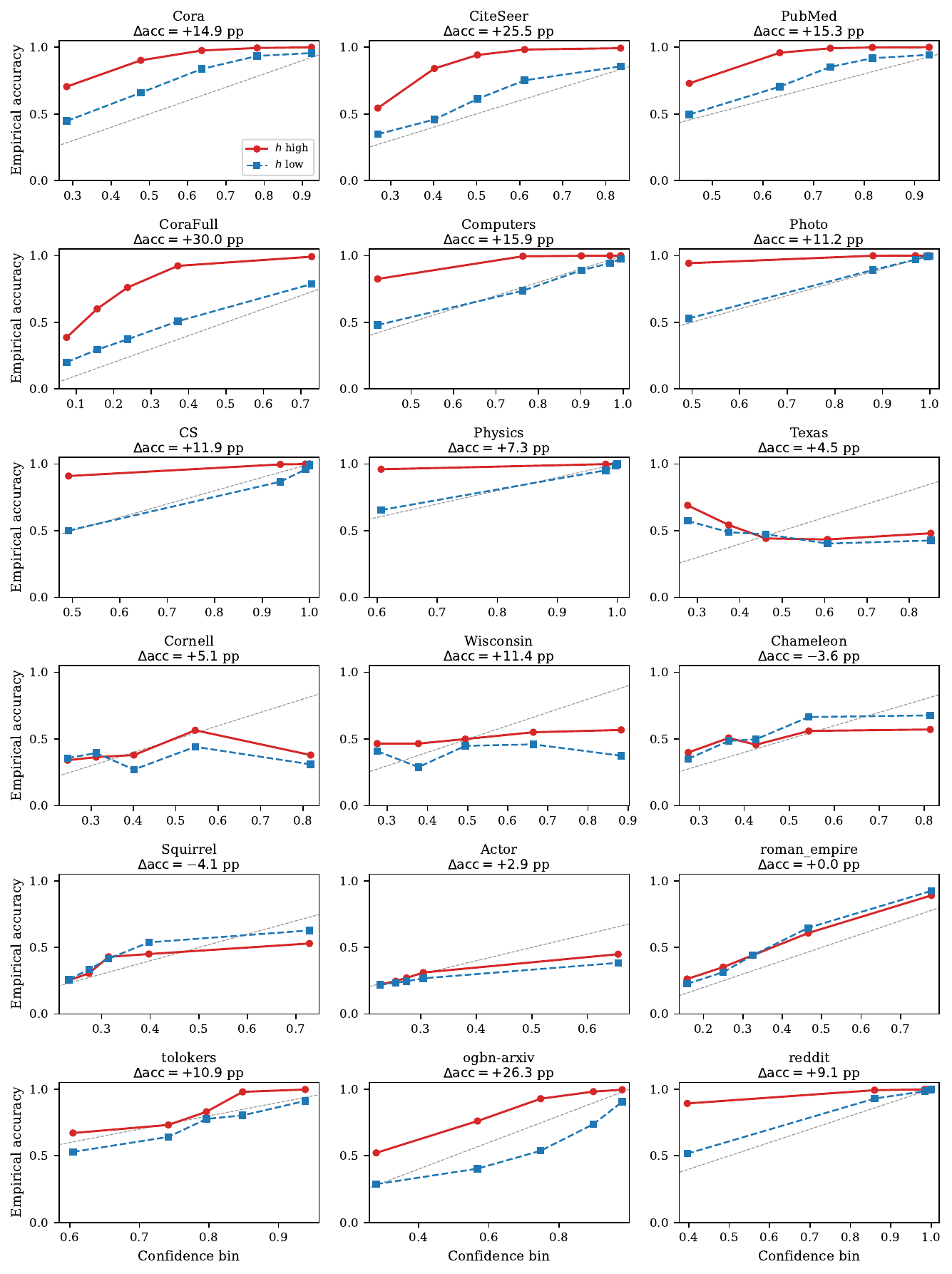}
\caption{Confidence-conditioned oracle-homophily diagnostic with a 50/50 split. For each dataset, test nodes are stratified by the base-model confidence and separated into low- and high-$h$ groups using a rank-based 50/50 split of oracle local homophily within each bin; ties retain the pooled node order. Oracle homophily is used to examine the structural reliability signal separately from estimation error. Curves show empirical accuracy within each confidence bin. The title of each panel reports the mean high-minus-low empirical accuracy gap $\Delta_{\mathrm{acc}}$ across the five confidence bins in percentage points.}
\label{fig:p1_homophily_split}
\end{figure}

\begin{figure}[h]
\centering
\includegraphics[width=\textwidth]{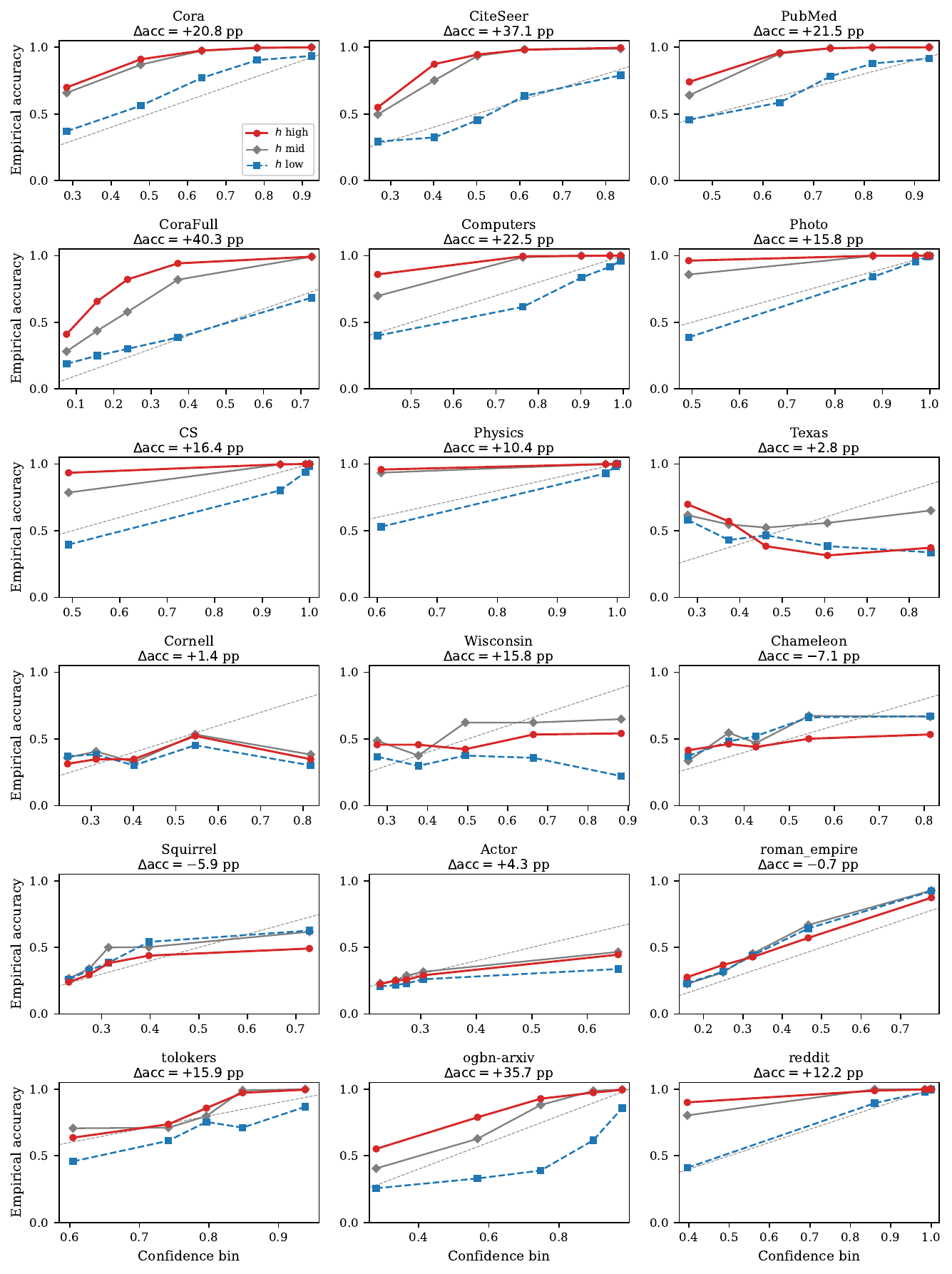}
\caption{Confidence-conditioned oracle-homophily diagnostic with three groups. This figure repeats the diagnostic in Figure~\ref{fig:p1_homophily_split}, but separates nodes by oracle-homophily rank into low-, middle-, and high-$h$ groups of approximately equal size within each confidence bin. The title of each panel reports the mean high-minus-low empirical accuracy gap $\Delta_{\mathrm{acc}}$ across the five confidence bins in percentage points. The three-way split visualizes how accuracy varies with neighborhood label composition at comparable confidence.}
\label{fig:p1_homophily_split_3bin}
\end{figure}

\subsection{Comparison with Other Structural Statistics}
\label{app:structural_comparison}

To examine the empirical basis for focusing on homophily, we compare its residual association with prediction correctness against representative motif, centrality, and curvature statistics. We use oracle normalized local homophily, the local clustering coefficient, PageRank, and mean incident augmented Forman--Ricci curvature. Oracle homophily is computed from true labels solely for this retrospective diagnostic; it is not an input to the HoTS predictor or temperature map.

We evaluate GCN predictions over five seeds. For each dataset and seed, we separately regress the correctness indicator and each structural statistic on an intercept, uncalibrated confidence, and $\log(1+d_i)$, where $d_i$ is the in-degree in the backbone input graph. Their residual correlation measures the additional linear association with correctness after controlling for confidence and degree. The pooled summary concatenates residuals across datasets and seeds; the macro summary averages dataset-level correlations, each computed after pooling that dataset's residuals across seeds.

Clustering, PageRank, and curvature are computed on the undirected, self-loop-free graph. PageRank uses damping $0.85$. For an edge $(u,v)$, augmented Forman--Ricci curvature is $4-d_u-d_v+3t_{uv}$, where degrees and the common-neighbor count $t_{uv}$ refer to this graph; node curvature averages over incident edges. Oracle homophily is the same-label outgoing-neighbor fraction on the backbone input graph, including self-loops, normalized as $\tilde h_i=(Kh_i-1)/(K-1)$. Homophily and PageRank cover all benchmarks; clustering and curvature omit Reddit because of the cost of triangle computations.

\begin{table}[ht]
\centering
\small
\caption{Partial correlation with prediction correctness after controlling for confidence and log-degree. Results use GCN and five seeds. Homophily and PageRank cover 18 datasets; clustering and curvature cover 17, excluding Reddit.}
\label{tab:structural_comparison}
\begin{tabular}{llrr}
\toprule
Structural statistic & Category & Pooled & Macro \\
\midrule
Oracle normalized homophily & Homophily & $+0.384$ & $+0.283$ \\
Local clustering coefficient & Motifs & $+0.044$ & $+0.012$ \\
PageRank & Centrality & $+0.003$ & $-0.003$ \\
Mean incident augmented Forman--Ricci curvature & Curvature & $-0.026$ & $-0.052$ \\
\bottomrule
\end{tabular}
\end{table}

Among the statistics examined, oracle homophily shows the strongest residual association with correctness in both summaries (Table~\ref{tab:structural_comparison}). Together with the explicit homophily-dependent temperature law from the CSBM analysis, this provides an empirical basis for our choice of structural variable. The comparison concerns the structural properties themselves and their conditional linear association with correctness.

\clearpage
\subsection{HoTS Component Ablation}
\label{app:component_ablation}

\paragraph{Component ablation.} Table~\ref{tab:ablation_summary} compares TS, entropy-only scaling, homophily-only scaling, HoTS with a fixed exponent, and the full HoTS. TS uses a global temperature. Entropy-only uses $T_i=T_{\mathrm{base}}+\beta\sqrt{2K\log K(1-e_i)}$, omitting the homophily denominator and fitting its own positive scalar parameters. Homophily-only uses $T_i=T_{\mathrm{base}}+\beta/(|\hat{\htilde}_i|+\varepsilon)^\alpha$, omitting the entropy numerator. HoTS combines both terms as defined in Section~\ref{sec:method}, while HoTS ($\alpha=1$) retains this combination but fixes the exponent. HoTS achieves the lowest reported mean ECE.

\begin{table}[!h]
\centering
\setlength{\tabcolsep}{4pt}
\caption{HoTS component ablation. Mean ECE (\%) across 18 datasets, pooling GCN and GAT with 10 seeds per backbone. HoTS and HoTS ($\alpha=1$) use the same observed-label-only homophily estimates. $^\dagger$Homophily-only retains the previously reported result and has not been rerun with the current observed-label-only estimator. Lower is better; best in \textbf{bold}.}
\label{tab:ablation_summary}
\resizebox{\textwidth}{!}{%
\begin{tabular}{lccccc}
\toprule
 & TS & Entropy-only & Homophily-only$^\dagger$ & HoTS ($\alpha=1$) & HoTS (ours) \\
\midrule
Mean ECE (\%) & $6.04$ & $4.83$ & $5.56$ & $5.90$ & $\boldsymbol{4.79}$ \\
\bottomrule
\end{tabular}%
}
\end{table}

\paragraph{Learned versus fixed exponent.} The fixed-$\alpha=1$ variant refits $T_{\mathrm{base}}$ and $\beta$ using the same backbone checkpoints, homophily estimates, splits, and optimization settings as HoTS. Its higher mean ECE supports learning the homophily exponent rather than using the oracle asymptotic value on noisy real graphs.

\clearpage
\section{Implementation Details}                                                                 
\label{app:implementation}                                                                     

This appendix collects the experimental and implementation details needed to reproduce our       
results, including dataset statistics and preprocessing, hardware and software environment,
training hyperparameters for the GNN backbones and the calibration methods, the architecture of  
the HoTS homophily predictor, and a measurement of prediction preservation across calibrators. 

\subsection{Dataset Statistics and Preprocessing}
\label{app:datasets}
                                                                                               
Table~\ref{tab:datasets} reports the size, feature dimension, number of classes, and node-level  
homophily for the 18 benchmark datasets used in our experiments. Homophily is computed as the    
mean ($\pm$ std) of node-level homophily $h_i = d_i^{-1}\sum_{j\in\cN(i)} \mathbf{1}\{y_j =      
y_i\}$ across all nodes (with $h_i = 0$ for isolated nodes).                                   

\begin{table}[h]                                                                                 
\centering
\caption{Dataset statistics. Homophily is the mean ($\pm$ std) of node-level homophily $h_i$ over
all nodes. All datasets use random 20/10/70 train/val/test splits regenerated for each seed.}   
\label{tab:datasets}
\begin{tabular}{lrrrrc}                                                                          
\toprule                                                                                       
Dataset & Nodes & Edges & Classes & Features & Homophily \\                                      
\midrule
\multicolumn{6}{l}{\textit{Citation (Planetoid)}} \\                                             
Cora & 2{,}708 & 13{,}264 & 7 & 1{,}433 & 0.871 $\pm$ 0.211 \\                                   
CiteSeer & 3{,}327 & 12{,}431 & 6 & 3{,}703 & 0.820 $\pm$ 0.245 \\                               
PubMed & 19{,}717 & 108{,}365 & 3 & 500 & 0.865 $\pm$ 0.217 \\                                   
\midrule                                                                                         
\multicolumn{6}{l}{\textit{Amazon / Coauthor}} \\                                                
Computers & 13{,}381 & 491{,}556 & 10 & 767 & 0.802 $\pm$ 0.239 \\                               
Photo & 7{,}487 & 238{,}087 & 8 & 745 & 0.849 $\pm$ 0.227 \\                                     
CS & 18{,}333 & 163{,}788 & 15 & 6{,}805 & 0.832 $\pm$ 0.239 \\                                  
Physics & 34{,}493 & 495{,}924 & 5 & 8{,}415 & 0.915 $\pm$ 0.174 \\                              
CoraFull & 18{,}800 & 144{,}170 & 70 & 8{,}710 & 0.677 $\pm$ 0.274 \\                            
\midrule                                                                                         
\multicolumn{6}{l}{\textit{WebKB}} \\                                                            
Texas & 183 & 492 & 5 & 1{,}703 & 0.530 $\pm$ 0.268 \\                                           
Cornell & 183 & 478 & 5 & 1{,}703 & 0.560 $\pm$ 0.273 \\                                         
Wisconsin & 251 & 750 & 5 & 1{,}703 & 0.565 $\pm$ 0.290 \\                                       
\midrule                                                                                         
\multicolumn{6}{l}{\textit{Wikipedia}} \\                                                        
Chameleon & 2{,}277 & 36{,}101 & 5 & 2{,}325 & 0.104 $\pm$ 0.229 \\                              
Squirrel & 5{,}201 & 217{,}073 & 5 & 2{,}089 & 0.089 $\pm$ 0.194 \\                              
Actor & 7{,}600 & 40{,}869 & 5 & 932 & 0.608 $\pm$ 0.319 \\                                      
\midrule                                                                                         
\multicolumn{6}{l}{\textit{Heterophilous Benchmark}} \\                                          
Roman-Empire & 22{,}662 & 65{,}854 & 18 & 300 & 0.046 $\pm$ 0.139 \\                             
tolokers & 11{,}758 & 1{,}038{,}000 & 2 & 10 & 0.634 $\pm$ 0.247 \\                              
\midrule                                                                                         
\multicolumn{6}{l}{\textit{Large-scale}} \\                                                      
ogbn-arxiv & 169{,}343 & 2{,}501{,}829 & 40 & 128 & 0.707 $\pm$ 0.272 \\                         
Reddit & 232{,}965 & 114{,}848{,}857 & 41 & 602 & 0.813 $\pm$ 0.230 \\                           
\bottomrule                                                                                      
\end{tabular}                                                                                    
\end{table}                                                                                      
                                                                                             
\paragraph{Preprocessing.} We apply three preprocessing steps before training:                   
\textbf{(i)} \textbf{Largest connected component (LCC).} For Amazon (Computers, Photo) and
Coauthor / Coauthor-Full (CS, Physics, CoraFull) we extract the largest connected component,     
following \citet{zhuang2025gets}, so that all nodes are reachable during message passing.      
\textbf{(ii)} \textbf{Undirected edges.} Directed graphs (in particular ogbn-arxiv) are converted
to undirected by adding reverse edges.                                                          
\textbf{(iii)} \textbf{Self-loops.} A self-loop is added at every node, which is standard for
GCN-style aggregation \citep{kipf2017semi} and ensures that a node's own features participate in 
the layer-wise update.                                                                         
                                                                                               
\paragraph{Splits and protocol uniformity.} For each (dataset, seed) we sample a fresh random 20/10/70
train/validation/test split. The validation split is used for both early stopping of the GNN
backbone and for fitting all post-hoc calibrators. The test split is used only for the reported
metrics. We use random splits across all 18 datasets to keep the calibration evaluation protocol
consistent across heterogeneous benchmarks; this also matches the protocol assumed in the public      
reference implementations of the graph-aware baselines (CaGCN, GATS, GETS, WATS). GATS is omitted on Reddit because it exceeds the 24~GB GPU memory budget under this protocol (Section~\ref{exp:setup}).

\subsection{Hardware and Software}                                                               
\label{app:hardware}                                                                             
                                                                                               
All experiments are run on a workstation with 4$\times$ NVIDIA RTX A5000 GPUs (24~GB each); each 
individual run uses a single GPU. The software stack is Python~3.10, PyTorch~2.1.0 with          
CUDA~12.1, DGL~2.4.0, PyTorch Geometric~2.7.0, and OGB~1.3.6. Random seeds for Python, NumPy,    
PyTorch, and DGL are fixed before each (dataset, backbone, seed) run, and we set
\texttt{CUBLAS\_WORKSPACE\_CONFIG=:4096:8} together with                                       
\texttt{torch.use\_deterministic\_algorithms} for reproducible CUDA kernels.

\paragraph{Memory limits.} GATS results on Reddit are reported as ``--'' because its attention-based per-node temperature module exceeds the $24$~GB GPU memory budget on this graph (Tables~\ref{tab:main_ece} and~\ref{tab:app_ece}).

\subsection{GNN Training Hyperparameters}                                                        
\label{app:gnn_hyperparams}                               
                                                                                               
The main experiments use two backbones: a 2-layer GCN \citep{kipf2017semi} and a 2-layer GAT            
\citep{velivckovic2017graph}. Both backbones are trained from scratch for each (dataset, seed)   
pair using the cross-entropy loss on the training split, with early stopping on the validation   
loss. The GAT uses 2 attention heads with attention and feature dropout both equal to the GNN  
dropout. Hyperparameters follow the per-dataset settings adopted by prior calibration          
work~\citep{zhuang2025gets} and group into three regimes summarized in
Table~\ref{tab:gnn_hyperparams}.

\paragraph{Dataset groups.} (G1)~Cora, CiteSeer, PubMed, Texas, Cornell, Wisconsin, Chameleon,   
Squirrel, Actor, Roman-Empire, tolokers, and Reddit. (G2)~Computers, Photo, CS, Physics,
CoraFull. (G3)~ogbn-arxiv.                                                                       
                                                        
\begin{table}[h]                                                                                 
\centering\small
\caption{GNN training hyperparameters by dataset group. Optimizer is Adam in all cases; both     
backbones train for at most 200 epochs; ``--'' under Patience indicates no early stopping (full  
200 epochs are run). The same configuration is used for GCN and GAT, with GAT additionally using
2 attention heads and feature/attention dropout equal to the listed Dropout. Group~G3            
(ogbn-arxiv) additionally uses BatchNorm between layers, following the OGB convention.}
\label{tab:gnn_hyperparams}                                                                    
\begin{tabular}{lcccccc}
\toprule                                                                                         
Group & Hidden & Layers & LR & Weight decay & Dropout & Patience \\
\midrule                                                                                         
G1 & 16  & 2 & $10^{-2}$ & $5\times 10^{-4}$ & 0.5 & 50 / -- \\
G2 & 64  & 2 & $10^{-2}$ & $10^{-3}$         & 0.5 & 50 \\                                       
G3 & 256 & 2 & $10^{-2}$ & $0$               & 0.5 & -- \\                                       
\bottomrule                                                                                      
\end{tabular}                                                                                    
\end{table}                                           
                                                                                               
\paragraph{Common settings.} All runs use the Adam optimizer, train for at most 200 epochs, and  
use the validation split both for early stopping and for fitting all post-hoc calibrators.       
ogbn-arxiv additionally uses BatchNorm between GCN/GAT layers, following the OGB leaderboard     
convention. PubMed, ogbn-arxiv, and Reddit run all 200 epochs without early-stopping patience; 
the remaining datasets stop when the validation loss has not improved for 50 epochs.

\subsection{Calibration Method Hyperparameters}                                                
\label{app:cal_hyperparams}                                                                      
                                                        
With the GNN backbone parameters frozen, the iterative calibration stage optimizes parameters
on the validation (calibration) split. We use the Adam optimizer with learning rate $10^{-2}$ and no weight decay,
train for at most $1{,}000$ epochs, and use training-split cross-entropy for checkpoint selection
and early stopping with a patience of $50$ epochs. Best weights from the early-stopping checkpoint are restored before
evaluation. The separate HoTS homophily predictor follows the supervision and training protocol in Appendix~\ref{app:hots_predictor}.
These common iterative calibration settings apply to all methods listed below; method-specific
architectural details are summarized in Table~\ref{tab:cal_hyperparams}.

\begin{table}[h]
\centering\small
\caption{Method-specific calibration settings. Common settings (Adam, $\text{lr}{=}10^{-2}$,
weight decay $0$, $1{,}000$ epochs, patience $50$) apply to all methods. ``Per-dataset config''  
indicates that learning rate, hidden width, attention heads, or other architectural choices
follow the per-dataset configuration files released with the original method.}                   
\label{tab:cal_hyperparams}                               
\begin{tabular}{lp{8.5cm}}                                                                     
\toprule                                                                                         
Method & Method-specific settings \\
\midrule                                                                                         
TS, ETS, HTS & Scalar parameters only; no architectural hyperparameters. \\
VS  & Per-class affine logit transform $z_c \mapsto w_c\,z_c + b_c$. \\                        
CaGCN, GATS & Per-dataset config (auxiliary GCN/attention width and dropout) following the       
original releases. \\                                                                            
GETS & Per-dataset config (mixture-of-experts width and number of experts) following the original
release. \\                                                                                     
WATS & Per-dataset config (graph wavelet basis size) following the original release. \\        
HoTS (ours) & Three scalar temperature parameters $(T_{\mathrm{base}}, \beta, \alpha)$, enforced through     
softplus parameterizations $T_{\mathrm{base}}=\operatorname{softplus}(\hat T)+0.1$,              
$\beta=\operatorname{softplus}(\hat\beta)+0.01$, $\alpha=\operatorname{softplus}(\hat\alpha)+0.01$
on unconstrained underlying parameters $\hat T, \hat\beta, \hat\alpha\in\mathbb{R}$. We         
initialize $\hat T=\hat\beta=0.5$ and $\hat\alpha=0$, giving effective initial values
$T_{\mathrm{base}}\approx 1.07$, $\beta\approx 0.98$, $\alpha\approx 0.70$. The stabilizer in the
homophily denominator is fixed to $\varepsilon=0.02$. The homophily predictor that produces $\hat h_i$ is trained separately; see Section~\ref{app:hots_predictor}. \\
\bottomrule
\end{tabular}
\end{table}

\subsection{HoTS Homophily Predictor}                               
\label{app:hots_predictor}                                                                       
                                                                                               
The HoTS denominator depends on the normalized estimate $\hat{\tilde h}_i=(K\hat h_i-1)/(K-1)$. We train an auxiliary GCN on homophily targets constructed solely from observed training and validation labels. The predictor is fitted once within each (dataset, backbone, seed) calibration run, and its predictions are cached before fitting the temperature map. The base classifier and its logits remain fixed throughout calibration.

\paragraph{Architecture.} The predictor is a 2-layer GCN of the form                             
\begin{equation*}
  \hat h_i = \sigma\!\left(\mathrm{GCN}_2\!\left(\mathrm{ReLU}\!\left(\mathrm{GCN}_1(x_i;      
A)\right);\, A\right)\right),                                                                    
\end{equation*}                                                                                
mapping from input feature dimension $F$ to a hidden width of $32$ and then to a scalar output   
passed through a sigmoid so that $\hat h_i \in (0,1)$. The same architecture is used for all 18  
datasets.                                                                                        
                                                                                               
\paragraph{Supervision target.} Let $\mathcal I_{\mathrm{known}}=\mathcal I_{\mathrm{train}}\cup\mathcal I_{\mathrm{val}}$ be the nodes with observed labels. Define the observed-label neighborhood and the supervised node set as
\begin{align*}
\mathcal N_L(i) &= \{j\in\mathcal N(i): j\in\mathcal I_{\mathrm{known}},\ j\ne i\}, \\
\mathcal M &= \{i\in\mathcal I_{\mathrm{known}}: |\mathcal N_L(i)|>0\}.
\end{align*}
For $i\in\mathcal M$, the masked supervision target is
\begin{equation*}
h_i^{\mathrm{mask}}=\frac{1}{|\mathcal N_L(i)|}
\sum_{j\in\mathcal N_L(i)}\mathbf 1\{y_j=y_i\}.
\end{equation*}
The target uses only observed-label neighbors, excluding self-loops. Known nodes without such neighbors are excluded from the predictor loss.

\paragraph{Predictor training.} The auxiliary GCN minimizes
\begin{equation*}
\mathcal L_{\mathrm{pred}}(\theta)=\frac{1}{|\mathcal M|}
\sum_{i\in\mathcal M}\bigl(\hat h_i(\theta)-h_i^{\mathrm{mask}}\bigr)^2
\end{equation*}
using Adam with learning rate $10^{-2}$ and weight decay $10^{-4}$ for at most $200$ epochs. Early stopping monitors this supervised training loss with a patience of $30$ epochs, and the best-loss checkpoint is restored. The GCN uses the full transductive graph with self-loops and node features to predict homophily for all nodes, including those without observed labeled neighbors.

\paragraph{Temperature fitting and label access.} After predictor training, we compute and cache $\hat h_i$ for every node and normalize it to $\hat{\tilde h}_i=(K\hat h_i-1)/(K-1)$. These estimates are held fixed while fitting $T_{\mathrm{base}},\beta,\alpha$ on the frozen classifier logits. The temperature map minimizes validation-split cross-entropy; training-split cross-entropy selects the checkpoint and controls early stopping. Test labels are used only for evaluation and retrospective diagnostics, not for fitting or checkpoint selection.

\subsection{Prediction Preservation Across Calibrators}
\label{app:vsgets}

A node-wise calibrator that preserves the predicted class is desirable: any improvement in
selective classification or downstream decisions should come from better confidence
\emph{ranking}, not from changing the underlying classifier. To quantify this, we measure, for
each calibrator and for each of the $360 = 18{\times}2{\times}10$ (dataset, backbone, seed)
experiments, whether at least one test node's argmax differs from the uncalibrated baseline.
Table~\ref{tab:prediction-change} reports the resulting rate. Methods that apply a positive
scalar temperature to logits ($T_i > 0$ for every node) are guaranteed to preserve $\arg\max$;
methods that learn class-specific or per-class affine transforms (VS, GETS) can reorder logit
components and change the predicted class.

\begin{table}[h]
\centering\small
\caption{Prediction change rates across all $360$ experiments ($18$ datasets $\times$ $2$
backbones $\times$ $10$ seeds). ``Preserves'' indicates whether the method theoretically
preserves the predicted class. ``Rate'' is the fraction of experiments in which at least one test
node's $\arg\max$ differs from the uncalibrated baseline. Lower is more conservative. $^*$ETS
can in principle reweight class probabilities but empirically learns mixture weights very close
to $(1, 0, 0)$. $^\dagger$CaGCN and GATS produce strictly positive $T_i$, but rare numerical edge
cases (near-tied logits) flip a small number of predictions. $^\ddagger$Reddit GATS is omitted
due to OOM during fitting (Section~\ref{app:hardware}); the rate for GATS is computed over the
remaining $340$ experiments.}
\label{tab:prediction-change}
\begin{tabular}{lp{6.5cm}cc}
\toprule
Method & Calibrator form & Preserves & Rate \\
\midrule
\multicolumn{4}{l}{\textit{Global methods}} \\
TS    & $\mathrm{softmax}(z/T)$, $T>0$ scalar & Yes & $0.0\%$ \\
ETS   & $\sum_k w_k\, p^{(k)}$, mixture & No$^{*}$ & $0.3\%$ \\
VS    & $z^{\mathrm{cal}}_c = w_c z_c + b_c$ per-class affine & No & $96.1\%$ \\
\midrule
\multicolumn{4}{l}{\textit{Node-wise methods}} \\
HTS   & $\mathrm{softmax}(z_i/T_i)$, $T_i = a + b\,e_i$, $T_i>0$ & Yes & $0.0\%$ \\
CaGCN & $\mathrm{softmax}(z_i/T_i)$, $T_i = \mathrm{GCN}(x, A)_i$ & Yes$^{\dagger}$ & $1.9\%$ \\
GATS  & $\mathrm{softmax}(z_i/T_i)$, $T_i = \mathrm{Attn}(x, A)_i$ & Yes$^{\dagger}$ &
$0.9\%^{\ddagger}$ \\
GETS  & $z^{\mathrm{cal}}_{i,c} = z_{i,c}\cdot \mathrm{softplus}(T_{i,c})$ & No & $98.6\%$ \\
WATS  & $\mathrm{softmax}(z_i/T_i)$ with wavelet-based $T_i>0$ & Yes & $0.0\%$ \\
HoTS  & $\mathrm{softmax}(z_i/T_i)$, $T_i = T_{\mathrm{base}} + \dfrac{\beta\sqrt{2K\log
K(1-e_i)}}{(|\hat{\tilde h}_i|+\varepsilon)^{\alpha}} > 0$ & Yes & $\mathbf{0.0\%}$ \\
\bottomrule
\end{tabular}
\end{table}

HoTS, like TS, HTS, and WATS, never changes the predicted class because $T_i > 0$ is enforced by
the softplus parameterization (Section~\ref{app:cal_hyperparams}). This means the gains in
Table~\ref{tab:main_ece}, Table~\ref{tab:nll_degece}, and Table~\ref{tab:selective} reflect
improved probability ranking rather than reclassification. In contrast, VS and GETS modify the
predicted class on roughly $96\%$ and $99\%$ of experiments respectively, so any
selective-classification advantage they may show would be confounded with prediction refinement
and we therefore exclude them from Table~\ref{tab:selective}.

\clearpage

\end{document}